\documentclass[11pt]{article}

\usepackage{amsmath}
\usepackage{graphicx}
\usepackage{url} 
\usepackage{authblk}
\usepackage{algorithm}
\usepackage{amsthm}
\usepackage{amsfonts}
\usepackage{amssymb}
\usepackage{algpseudocode}
\usepackage{bbm}
\usepackage{bm}
\usepackage{booktabs}
\usepackage{caption}
\usepackage{color}
\usepackage{enumerate}
\usepackage{enumitem}
\usepackage{float}
\usepackage{mathtools} 
\usepackage{xr-hyper}
\usepackage{hyperref}
\usepackage{cleveref}
\usepackage{mathtools}
\usepackage{longtable}

\renewcommand{\mid}{\,|\,}

\newtheorem{lemma}{Lemma}
\newtheorem{corollary}{Corollary}
\newtheorem{proposition}{Proposition}
\newtheorem{theorem}{Theorem}

\theoremstyle{definition}
\newtheorem{example}{Example}

\theoremstyle{remark}
\newtheorem*{remark}{Remark}
 
\newcommand{\Var}{\mathrm{Var}}

\DeclareMathOperator*\argmin{arg\,min}
\DeclareMathOperator*\argmax{arg\,max}

\newcommand{\ind}{\mathbbm{1}}

\newcommand{\y}{\widetilde{y}}
\newcommand{\pt}{\mathsf{p}}

\newcommand{\cE}{\mathcal{E}}
\newcommand{\sE}{\mathsf{E}}
\newcommand{\bbE}{\mathbb{E}}

\newcommand{\cF}{\mathcal{F}}

\newcommand{\cI}{\mathcal{I}}

\newcommand{\cM}{\mathcal{M}}

\newcommand{\sd}{\mathsf{d}}

\usepackage[round]{natbib} 

\title{Adaptive Testing for Segmenting Watermarked Texts From Language Models \thanks{Accepted for publication in STAT, October 28, 2025.}}

\author{Xingchi Li$^{1,\dagger}$, Xiaochi Liu$^{2,}\footnote{Equal contribution}$, Guanxun Li$^{2,}\footnote{Corresponding author: guanxun@bnu.edu.cn}$ \smallskip \\
$^1$ Department of Statistics, Texas A\&M University\\
$^2$ Department of Statistics, Beijing Normal University at Zhuhai}

\date{}
\begin{document}

\maketitle

\begin{abstract}
The rapid adoption of large language models (LLMs), such as GPT-4 and Claude 3.5, underscores the need to distinguish LLM-generated text from human-written content to mitigate the spread of misinformation and misuse in education. One promising approach to address this issue is the watermark technique, which embeds subtle statistical signals into LLM-generated text to enable reliable identification. In this paper, we first generalize the likelihood-based LLM detection method of a previous study by introducing a flexible weighted formulation, and further adapt this approach to the inverse transform sampling method. Moving beyond watermark detection, we extend this adaptive detection strategy to tackle the more challenging problem of segmenting a given text into watermarked and non-watermarked substrings. In contrast to the approach in a previous study, which relies on accurate estimation of next-token probabilities that are highly sensitive to prompt estimation, our proposed framework removes the need for precise prompt estimation. Extensive numerical experiments demonstrate that the proposed methodology is both effective and robust in accurately segmenting texts containing a mixture of watermarked and non-watermarked content. \\

\end{abstract}

\section{Introduction} \label{sec:intro}
The rapid advancement of large language models (LLMs), such as GPT-4 and Claude 3.5, has facilitated their widespread use in text generation across diverse applications such as paper polish, knowledge integration, language translation, and code generation. However, the widespread adoption of LLMs has introduced substantial challenges, including the spread of misinformation, misuse in educational settings and the contamination of training data for subsequent models. Consequently, it has become increasingly important to develop reliable methods for detecting LLM-generated text and distinguishing it from human-written content.

To address these challenges, watermarking has emerged as a promising solution. By embedding subtle statistical signals into LLM-generated text, watermarking allows for reliable identification of LLM-generated content. In this framework, a third-party user sends prompts to a trusted LLM provider, which embeds a watermark into the generated text. The user may subsequently modify the text through substitutions, insertions, deletions, or paraphrasing before publication. Despite such edits, the watermark remains detectable. A central problem in this setting is to determine whether a given text is watermarked and, if so, to accurately detect and localize the watermarked sub-strings. Figure~\ref{fig:flow} illustrates the interactions and roles of the LLM provider, the user, and the detector.

\begin{figure}
    \centering
    \includegraphics[width=0.85\linewidth]{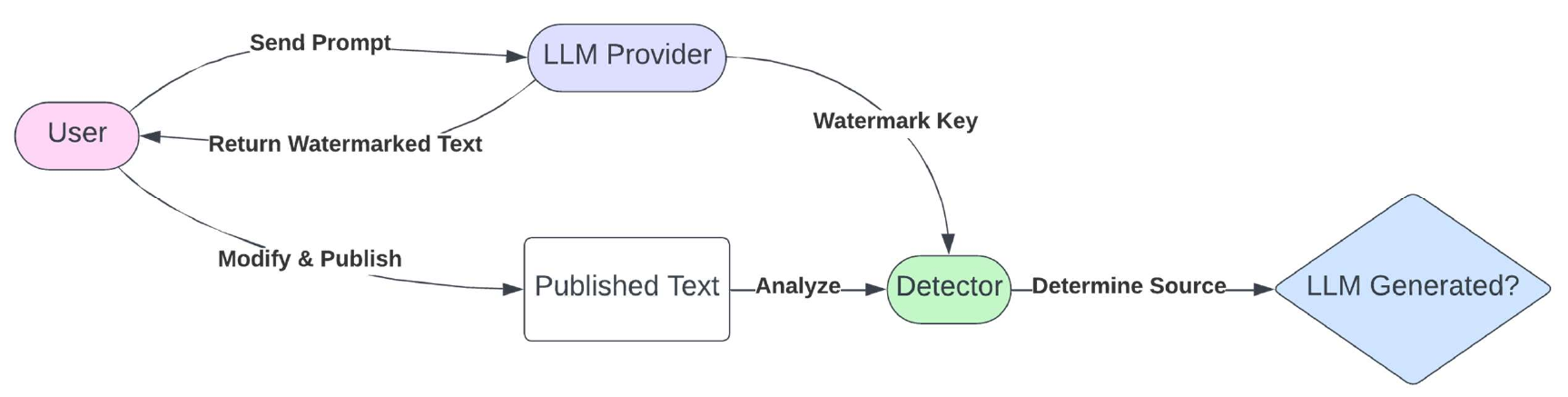}
    \caption{Illustration of the roles of the LLM provider, user and detector.}
    \label{fig:flow}
\end{figure}

Early text watermarking techniques primarily utilized post-processing methods, embedding watermarks after text generation \citep{liu2024survey}. In contrast, more sophisticated approaches integrate watermarking directly into the generation process. A seminal watermarking scheme proposed by \citet{kirchenbauer2023watermark}, known as the `red-green list' method, partitions vocabulary into two distinct categories, enhancing the probability of selecting tokens from the `green' list during next-token prediction. \citet{liu2024adaptive} improved upon this method by adaptively applying watermarking only to token distributions with high entropy. Additional studies exploring related watermarking schemes include \citet{kirchenbauer2023reliability, lee2023wrote, liu2023unforgeable, zhao2023protecting, cai2024towards, huo2024token, nemecek2024topic, zhou2024bileve}. Nonetheless, such watermarking techniques inherently introduce bias by altering next-token prediction distributions, potentially degrading text quality, especially in long or multiple watermarked responses.

In recent years, a series of unbiased watermarking methods have been developed. \citet{Aaronson2023} proposed the Exponential Minimum Sampling (EMS) method, which embeds a detectable watermark signal into generated text by multiplying all tokens by an exponential factor. This method is closely related to the Gumbel trick \citep{papandreou2011perturb} widely used in machine learning. The Inverse Transform Sampling (ITS) watermarking approach introduced by \citet{kuditipudi2023robust} ensures robustness to perturbations while preserving the original text distribution. Further unbiased methods and variations of the Gumbel trick are discussed in \citet{christ2023undetectable, fernandez2023three, hu2023unbiased, wu2023dipmark, wuresilient, fu2024gumbelsoft, zhao2024permute}. Recently, statistical analysis of watermark generation and detection has garnered increased attention. \citet{huang2023towards} were among the first to frame watermark detection as a hypothesis testing problem, analysing minimax Type II errors and identifying optimal testing procedures under the simplified assumption of independent and identically distributed tokens. \citet{li2024statistical} developed a statistical framework for evaluating the asymptotic efficiency of watermark detection, assuming an idealized scenario of fully human-written or fully LLM-generated texts. \citet{li2024robust} extended this framework by considering the detection of watermarked text potentially edited by humans. \citet{xie2025debiasing} proposed a method employing maximal coupling to debias the red-green list approach, establishing asymptotic detection boundaries. Comprehensive reviews of watermarking techniques in artificial intelligence (AI) can be found in \citet{zhao2024sok} and \cite{liu2024survey}.

Recent work by \citet{li2025likelihood} introduced a likelihood-ratio-based method for detecting LLM-generated text, which is particularly effective for watermark detection in short contexts. This approach provides a rigorous analysis of Type I and Type II errors in watermark detection tests. In this paper, we first generalize their framework to a more flexible weighted version and further adapt it to the ITS method. We then address a more challenging task—segmenting published text into watermarked and non-watermarked substrings—a problem that has received limited attention thus far. \citet{kashtan2023information} use log-perplexity as the test statistic and propose a method for identifying sentences most likely written by humans. Building on the framework of \citet{li2024segmenting}, we formulate this task as a change-point detection problem. Specifically, we partition the text into overlapping substrings of a predetermined length and sequentially test each for the presence of a watermark using a randomization-based test. This process yields a sequence of $p$-values, which are subsequently analysed to detect changes in the underlying distribution. Watermarked segments typically produce $p$-values concentrated near zero, while non-watermarked segments yield $p$-values uniformly distributed between 0 and 1. Identifying these change points allows for effective segmentation of the text into watermarked and non-watermarked regions.

The $p$-value sequence is generated using our adaptive method. While the original approach requires accurate estimation of next-token probabilities (NTPs) \citep{li2025likelihood}, which are highly dependent on precise prompt estimation, our framework removes this requirement. In practice, accurately estimating prompts is even more challenging than estimating NTPs, and obtaining precise prompt estimates is typically infeasible in real-world scenarios. To overcome this limitation, we divide the text into small segments, allowing the preceding tokens to naturally serve as prompts and provide accurate NTP estimates for each segment. Since the true prompt only influences the initial segment, we simply use an empty set as the initial prompt, making precise prompt estimation unnecessary.

We validate the effectiveness of our approach through extensive numerical experiments on texts generated by two language models: \verb|Meta-Llama-3-8B| from Meta and \verb|Mistral-7B-v0.1| from Mistral AI, using prompts extracted from Google's C4 dataset. Experimental results demonstrate that our proposed segmentation method more accurately localizes watermark-induced change points than baseline approaches, highlighting the practical benefits of our method.

The remainder of this paper is organized as follows. Section~\ref{sec:prob} introduces the watermark detection problem and the randomization-based test. Section~\ref{sec:ada-stat} presents the adaptive test statistic, and Section~\ref{sec:cpd} describes the segmentation method for detecting watermarked text. Section~\ref{sec:numerical-experiments} reports the numerical results, demonstrating the efficiency of the proposed approach. Finally, Section~\ref{sec:conc} concludes the paper, with proofs and additional results provided in the Supporting Information.

\section{Problem setup}\label{sec:prob}
\subsection{Watermarked Text Generation}
Let $\mathcal{V}$ denote a finite vocabulary set of size $V = |\mathcal{V}|$. For simplicity, we assign a unique index from the set $[V] \coloneqq \{1, 2, \dots, V\}$ to each token in $\mathcal{V}$. Consider an autoregressive LLM $P$, which maps a string $y_{-n_0:t-1} = y_{-n_0} y_{-n_0+1} \cdots y_{t-1} \in \mathcal{V}^{t + n_0}$ to a probability distribution over $\mathcal{V}$. Here, $y_{-n_0:0}$ represents the user-provided prompt. Let $\mu_t(\cdot) \coloneqq p(\cdot \mid y_{-n_0:t-1})$ denote the sampling distribution of the token at time $t$. Suppose a string $y_{1:n}$ is generated by the LLM. We define $p_t \coloneqq \mu_t(y_t)$ as the probability of generating token $y_t$ given the preceding tokens $y_{-n_0:t-1}$ (also known as the NTP).

Let $\xi_{1:t} = \xi_1 \xi_2 \cdots \xi_t$ be the watermark key sequence, where $\xi_i \in \Xi$ with $\Xi$ representing the general key space. When a third-party user provides a prompt, the LLM provider uses a generator to produce text autoregressively via a decoder function $\Gamma$. Specifically, $\Gamma$ maps the current watermark key $\xi_t$ and the distribution $\mu_t$ to a token in $\mathcal{V}$. In this paper, we focus on an unbiased watermarking scheme where the original text distribution is preserved:
\begin{equation}\label{eq:unbiased-wm}
P(\Gamma(\xi_t, p_t) = y) = p_t(y).    
\end{equation}
The watermark text generation algorithm recursively constructs the generated string $y_{1:n}$ as $y_i = \Gamma(\xi_i, p(\cdot \mid y_{-n_0:i-1}))$ for $1 \leq i \leq n$, where $n$ is the number of tokens, and each $\xi_i$ is independently sampled from a distribution $\nu$ over $\Xi$. This process ensures that each token $y_i$ is uniquely determined by $\xi_i$ and the preceding context $y_{-n_0:i-1}$, given $p(\cdot \mid y_{-n_0:i-1})$.

Below, we introduce two popular unbiased watermarking schemes: the EMS technique by \citet{Aaronson2023} and the ITS method proposed by \citet{kuditipudi2023robust}.

\begin{example}[EMS]\label{ex1}
To generate the $i$th token, we first independently sample $\xi_{ik} \sim \text{Unif}[0,1]$ for each $1 \leq k \leq V$, denoting $\xi_i = [\xi_{i1}, \cdots, \xi_{iV}]$. Subsequently, the token $y_i$ is determined as:
\[y_i = \argmax_{k} \frac{\log(\xi_{ik})}{p(k | y_{-n_0:i-1})} = \argmin_{k} \frac{-\log(\xi_{ik})}{p(k | y_{-n_0:i-1})} = \argmin_{k} E_{ik},\]
where $E_{ik} = -\log(\xi_{ik}) / p(k | y_{-n_0:i-1}) \sim \text{Exp}(p(k | y_{-n_0:i-1}))$. Using properties of the exponential distribution, we can verify that $P(y_i = k) = P\bigl(E_{ik} < \min_{j \neq k} E_{ij}\bigr) = p(k | y_{-n_0:i-1})$, which confirms the unbiased nature of this generation scheme.

For a string $\y_{1:n}$ (of the same length as the key) that may be watermarked, \citet{Aaronson2023} propose measuring the dependence between the string $\y_{1:n}$ and the key sequence $\xi_{1:n}$ using the metric:
\begin{align}\label{eq-stat1}
\phi(\xi_{1:n},\y_{1:n})=\frac{1}{n}\sum_{i=1}^n\log(\xi_{i, \y_i}).
\end{align}
If $\y_i$ is generated using the scheme above with the key $\xi_i$, then $\xi_{i, \y_i}$ tends to be larger than the other components of $\xi_i$. Thus, a higher value of $\phi$ suggests that $\y_{1:n}$ was likely generated with watermarking.
\end{example}

\begin{example}[ITS]\label{ex2}
Let $\mu_i(k) = p(k | y_{-n_0:i-1})$ for $1 \leq k \leq V$ and $1 \leq i \leq n$. To generate the $i$th token, we first generate a permutation $\pi_i$ on $[V]$ and sample $u_i \sim \text{Unif}[0,1]$, thereby forming the key $\xi_i = (\pi_i, u_i)$. The decoder function $\Gamma$ is defined as 
\[\Gamma((\pi_i, u_i), \mu_i) = \pi_i^{-1}\bigl(\min \left\{ \pi_i(l) \; \middle| \; \mu_i(j : \pi_i(j) \leq \pi_i(l)) \geq u_i \right\}\bigr).\]
We note that $\Gamma((\pi_i, u_i), \mu_i) = k$ if $\min\{\pi_i(l) : \mu_i(j : \pi_i(j) \leq \pi_i(l)) \geq u_i\} = \pi_i(k)$, which implies that $\mu_i(j : \pi_i(j) \leq \pi_i(k)) \geq u_i > \mu_i(j : \pi_i(j) < \pi_i(k))$. As the length of this interval is $\mu_i(k)$, it follows that $P(\Gamma((\pi_i, u_i), \mu_i) = k) = \mu_i(k)$.

For a string $\y_{1:n}$ (of the same length as the key) that may be watermarked, \citet{kuditipudi2023robust} quantify the dependence between $\y_{1:n}$ and $\xi_{1:n}$ using the metric:
\begin{equation}\label{eq-stat2}
\phi(\xi_{1:n}, \y_{1:n}) = \frac{1}{n} \sum_{i=1}^n (u_i - 1 / 2) \left( \frac{\pi_i(\y_i) - 1}{V - 1} - \frac{1}{2} \right).
\end{equation}
If $\y_i$ is generated using the ITS scheme with the key $\xi_i = (\pi_i, u_i)$, then $u_i$ and $\pi_i(\y_i)$ are positively correlated. Consequently, a large value of $\phi$ suggests that $\y_{1:n}$ is likely watermarked.
\end{example}

\subsection{Watermarked Text Detection}\label{sec:dect}
In this section, we address the detection problem of determining whether a given text is watermarked. In practice, the published text $\y_{1:m}$ may differ from the original text $y_{1:n}$ generated by the LLM using the key $\xi_{1:n}$. Let $\mathcal{T}$ be a transformation function such that $\y_{1:m} = \mathcal{T}(y_{1:n})$. This transformation may involve substitutions, insertions, deletions, paraphrasing, or other edits. Detecting watermarks in such transformed text is challenging because the transformation $\mathcal{T}$ can obscure the watermarking signal.

Consider a published string $\y_{1:m}$ generated by a third-party user and a key sequence $\xi_{1:n}$ provided to the detector. The detector performs a hypothesis test:
\[H_{0}: \text{$\y_{1:m}$ is not watermarked} \quad \text{vs.} \quad H_{a}: \text{$\y_{1:m}$ is watermarked},\]
by computing a $p$-value based on the test statistic $\phi(\xi_{1:n}, \y_{1:m})$.

The test statistic $\phi$ measures the dependence between $\y_{1:m}$ and $\xi_{1:n}$. We assume that larger values of $\phi$ provide stronger evidence against the null hypothesis. To compute the $p$-value, we use a randomization test as follows. Generate $T$ independent key sequences $\xi^{(t)} = \xi^{(t)}_1 \cdots \xi^{(t)}_n$, where $\xi_{i}^{(t)} \sim \nu$ for $1 \leq t \leq T$ and $1 \leq i \leq n$, with each $\xi_{i}^{(t)}$ independent of $\y_{1:m}$. The randomization-based $p$-value is then defined as:
\[\pt_T = \frac{1}{T+1} \left( 1 + \sum_{t=1}^{T} \ind\left\{ \phi(\xi_{1:n}, \y_{1:m}) \leq \phi(\xi_{1:n}^{(t)}, \y_{1:m}) \right\} \right).\]
Given a pre-specified level $\alpha \in (0,1)$, we reject the null hypothesis $H_0$ if $\pt_T \leq \alpha$. 

To analyse the Type I and Type II errors of this testing procedure, let $\cF_m = [y_{-n_0:0}, \y_{1:m}]$. Given $\phi(\xi_{1:n}, \y_{1:m})$, define $\sd_{\xi} \coloneqq \sd(\phi(\xi_{1:n}, \y_{1:m}) \mid \cF_m)$ and $\sE_{\xi} \coloneqq \bbE[\phi(\xi_{1:n}, \y_{1:m}) \mid \cF_m]$. Similarly, let $\sd_{\xi'}$ and $\sE_{\xi'}$ be defined analogously with $\xi_{1:n}$ replaced by $\xi'_{1:n}$, where $\xi'_{1:n}$ is a key sequence generated independently of $\y_{1:m}$ in the same manner as $\xi_{1:n}$.

The following proposition proved by \citet{li2025likelihood} shows that the proposed randomization test controls both Type I and Type II errors under suitable conditions.
\begin{proposition}\label{thm:error-control}(Theorem 1 of \citealt{li2025likelihood})
For the randomization test, the following holds:
\begin{itemize}
    \item[(i)] Under the null hypothesis,
    \[P(\pt_T \leq \alpha) = \lfloor (T+1)\alpha \rfloor / (T+1) \leq \alpha,\]
    where $\lfloor a \rfloor$ denotes the greatest integer less than or equal to $a$.
    \item[(ii)] For any $\epsilon > 0$, if $T > 2/\epsilon - 1$, $\sE_{\xi'} = o(\sE_{\xi})$, $\sd_{\xi'} = o(\sE_{\xi})$, and $\sd_{\xi} = o(\sE_{\xi})$, then
    \begin{equation}\label{eq-power}
    P(\pt_T \leq \alpha \mid \cF_n) \geq 1 - C_1 \exp(-2T\epsilon^2) + o(1),
    \end{equation}
    as $n \to \infty$, where $C_1 > 0$.
\end{itemize}
\end{proposition}

By applying Proposition \ref{thm:error-control} to Examples \ref{ex1} and \ref{ex2} with $\phi$ defined in \eqref{eq-stat1} and \eqref{eq-stat2} and $\y_{1:m} = y_{1:n}$ for $n=m$, we identify conditions under which the power of the randomization test for these two watermarking schemes converges to 1 as $T \to \infty$.

\begin{corollary}\label{cor:thm-exp}
If $\y_{1:n} = y_{1:n}$ and
\begin{equation}\label{eq:cond-examp}
\cE \coloneqq \frac{1}{\sqrt{n}} \sum_{i=1}^n \left(1 - p(y_i \mid y_{-n_0:i-1})\right) \to \infty,
\end{equation}
then the conditions of Theorem \ref{thm:error-control} are satisfied for Examples \ref{ex1} and \ref{ex2}.
\end{corollary}

The quantity $\cE$ is referred to as the watermark potential by \citet{kuditipudi2023robust}, and it characterizes the entropy of the LLM. This measure indicates the detector’s ability to distinguish watermarked text from unwatermarked text. For instance, if the language model is deterministic—producing the same output for a given input without randomness—then $\cE = 0$, as $p(y_i \mid y_{-n_0:i-1}) = 1$ for all $i$.  In this case, it becomes impossible to differentiate watermarked text from unwatermarked text.

When the published text $\y_{1:m}$ is modified, not all tokens will maintain correlation with the key sequence. Instead, certain substrings are expected to exhibit correlation under $H_a$. To detect this, a scanning method is employed that compares every segment of length $B$ in $\y_{1:m}$ to each segment of length $B$ in $\xi_{1:n}$. Define the dependence measure $\phi(\xi_{a:a+B-1}, \y_{b:b+B-1})$ according to the chosen watermarking method. We consider the maximum test statistic defined as
\begin{equation}\label{eq:max-stat}
\cM(\xi_{1:n}, \y_{1:m}) = \max_{1 \leq a \leq n-B+1} \max_{1 \leq b \leq m-B+1} \phi(\xi_{a:a+B-1}, \y_{b:b+B-1}).    
\end{equation}

\begin{proposition}\label{prop-max}
Consider the maximum statistic defined above, where the dependence measure is
\[\phi(\xi_{a:a+B-1}, \y_{b:b+B-1}) = \frac{1}{B} \sum_{i=0}^{B-1} h_i(\xi_{a+i}, \y_{b+i}),\]
with $h_i$s are independent given $[\y_{1:m}, y_{-n_0:n}]$. If
\[C_{N,B}^{-1} \max_{a,b} \mathbb{E}\left[ \mathcal{M}(\xi_{a:a+B-1}, \y_{b:b+B-1}) \mid \y_{1:m}, y_{-n_0:n} \right] \to \infty,\]
where $N = \max\{n, m\}$ and $C_{N,B} = \log(N) / B$ in Example~\ref{ex1}, and $C_{N,B} = \sqrt{\log(N) / B}$ in Example~\ref{ex2}, then \eqref{eq-power} holds.
\end{proposition}

\section{Adaptive Test Statistics for Watermark Detection}\label{sec:ada-stat}
Recall that the NTP is defined as $p_i \coloneqq p(y_i \mid y_{-n_0:i-1})$. It is evident that tokens with smaller $p_i$ contribute more substantially to the quantity $\cE$ in \eqref{eq:cond-examp} than those with larger $p_i$. In other words, tokens with lower $p_i$ play a more critical role in distinguishing watermarked text from unwatermarked text. Motivated by this observation, \citet{li2025likelihood} proposed assigning greater weights to tokens with smaller $p_i$ in the test statistics. In this section, we extend this idea to the ITS method.

\subsection{Adaptive Test Statistics for the EMS Method}
\citet{li2025likelihood} introduced the following general class of test statistics for the EMS method:
\begin{equation}\label{eq:gene-stat1}
\phi(\xi_{1:n}, \y_{1:n};w_{1:n}) = \frac{1}{n}\sum_{i=1}^n w_i \log(\xi_{i, \y_i}),
\end{equation}
where $w_{1:n} = (w_1, \dots, w_n)$ are nonnegative weights. Setting $w_i = 1$ for all $i$ recovers the original statistic in \eqref{eq-stat1}.

In \citet{li2025likelihood}, the weights $w_i$ in \eqref{eq:gene-stat1} are chosen based on the likelihood ratio (LR) test. According to Lemma 1 in \citet{li2025likelihood}, under the EMS framework,
\[-\log(\xi_{i, \y_i}) \mid y_{-n_0:i-1}, \y_{1:i-1} \sim 
\begin{cases}
\mathrm{Exp}(1) & \text{if $\y_i$ is not generated from $\xi_i$}, \\
\mathrm{Exp}(1 / p_i) & \text{if $\y_i$ is generated from $\xi_i$}.
\end{cases}\]
Thus, the log-likelihood ratio test statistic for assessing whether $\y_i$ is generated from $\xi_i$ is
\begin{align}\label{eq:lrt-stat1}
L(\xi_{1:n}, \y_{1:n}) = \frac{1}{n} \sum_{i=1}^n \frac{1 - p_i}{p_i} \left( \log(\xi_{i, \y_i})\right).
\end{align}
Notably, tokens with smaller NTPs receive larger weights in the LR statistic, thereby amplifying their influence compared to the original statistic. Consequently, even a small number of tokens with low NTPs can substantially enhance the detection power of the LR test. In \citet{Aaronson2023} and \citet{kuditipudi2023robust}, the authors proposed using $-\log(1 - \xi_{i, \y_i})$ in practice, which demonstrated superior performance in watermark detection. In our sub-string identification problem later, however, we found no significant performance differences between the two test statistics. For completeness, we report the results based on $-\log(1 - \xi_{i, \y_i})$ in Supporting Information B.

It is worth noting that the weighting approach can be applied to any watermarking technique whose test statistic admits a summation form, such as the ITS and SynthId methods \citep{dathathri2024scalable}. In the following section, we demonstrate how this idea can be adapted to the ITS method.

\subsection{Adaptive Test Statistic for the ITS Method}
\subsubsection{A General Form of the Adaptive Test Statistic under ITS Framework}
Motivated by the general formulation in \eqref{eq:gene-stat1}, we introduce the following class of test statistics tailored to the ITS method:
\begin{equation}\label{eq:gene-stat2}
\phi(\xi_{1:n}, \y_{1:n};w_{1:n}) = \frac{1}{n}\sum_{i=1}^n w_i (u_i - 1 / 2)\left(\frac{\pi_i(\y_i) - 1}{V - 1} - \frac{1}{2}\right),
\end{equation}
where $w_{1:n} = (w_1, \dots, w_n)$ are nonnegative weights. Setting $w_i = 1$ for all $i$ reduces \eqref{eq:gene-stat2} to the original ITS statistic \eqref{eq-stat2}.

The following corollary provides a sufficient condition for asymptotic full power of the randomization test based on \eqref{eq:gene-stat2} under the ITS framework.

\begin{corollary}\label{cor:its}
If the weights satisfy
\begin{equation}\label{eq:cond-genestat}
\frac{\sum_{i=1}^{n} w_i (1 - p_i)}{\sqrt{\sum_{i=1}^{n} w_i^2}} \to +\infty,  
\end{equation} 
then the conditions of Theorem~\ref{thm:error-control} hold for the statistic defined in \eqref{eq:gene-stat2}. Consequently, the power of the randomization-based test converges to one as $T \to +\infty$ under the ITS framework.
\end{corollary}

The theorem below generalizes Proposition~\ref{prop-max} by incorporating the weighted ITS statistic.
\begin{theorem}\label{thm:its}
Consider the maximum statistic defined in \eqref{eq:max-stat} with $\phi$ given by \eqref{eq:gene-stat2} in the ITS setting. If
\[C_{N,B}^{-1} \max_{a,b} \mathbb{E}\left[ \phi\left( \xi_{a:a+B-1}, \y_{b:b+B-1} \right) \,\middle|\, \y_{1:m}, y_{-n_0:n} \right] \to +\infty, \]
where $N = \max\{n, m\}$, $C_{N,B} = \sqrt{\Omega_{\max}} \sqrt{\log(N) / B}$, and $\Omega_{\max} = \max_{i \in [n]} w_i^2$, then the power condition \eqref{eq-power} is satisfied.
\end{theorem}

Although the weight $w_i$ for ITS cannot be derived from the likelihood ratio, it is still possible to use the choice $w_i = (1 - p_i) / p_i$. However, empirical results indicate that the weighted test statistic in \eqref{eq:gene-stat2} performs poorly in practical settings. Consequently, we propose a new approach for constructing an adaptive test statistic specifically designed for the ITS framework.

\subsubsection{Robust Adaptive Test Statistic Using Huber's Contamination Model}
We first clarify why the LR test statistic cannot be directly applied in the ITS setting. Define $\mu_i(k) = p(k \mid y_{-n_0:i-1})$ for $1 \leq k \leq V$ and $1 \leq i \leq n$, and introduce intervals $\cI_i \coloneqq [\mu_i(j : \pi_i(j) < \pi_i(k)), \mu_i(j : \pi_i(j) \leq \pi_i(k))]$. By the construction of the ITS method, we obtain the following lemma:
\begin{lemma}\label{lemma:its}
Under the ITS framework, the conditional distribution of $u_i$ is given by
\[u_i \mid y_{-n_0:i-1} \sim 
\begin{cases}
\text{Unif}[0, 1] & \text{if $y_i$ is not generated from $\xi_i$}, \\
\text{Unif}[\cI_i] & \text{if $y_i$ is generated from $\xi_i$}.
\end{cases}\]
\end{lemma}

The likelihood ratio for testing the hypothesis $H_{i0}\colon y_i$ is not generated from $\xi_i$ against the alternative $H_{i1}\colon y_i$ is generated from $\xi_i$ is thus $\Lambda_i =  1 / p_i$ if $u_i \in \cI_i$ and $\Lambda_i = 0$ otherwise. Therefore, the overall likelihood ratio statistic $\Lambda \coloneqq \prod_{i=1}^n \Lambda_i$ is equal to $\prod_{i=1}^n 1 / p_i $ if all $u_i \in \cI_i$ and $\Lambda = 0$ otherwise. Since $\Lambda \neq 0$ only if all $u_i \in \mathcal{I}_i$, implying all tokens $y_{1:n}$ are generated from $\xi_{1:n}$. In practical applications with modified text $\y_{1:m}$, expecting all tokens to retain alignment with the generated sequence is unrealistic. Consequently, the LR test statistic always equals zero in practice, rendering it ineffective.

To overcome this limitation, we propose using Huber's $\epsilon$-contamination model \citep{huber1973robust} under the alternative hypothesis: $u_i \sim (1-\epsilon)\text{Unif}[\cI_i] + \epsilon\text{Unif}[0, 1]$, where $\epsilon \in [0, 1]$ denotes the contamination proportion. Under this assumption, we obtain the likelihood ratio as $\Lambda_i = (1 - \epsilon) / p_i\ind\{u_i \in \cI_i\} + \epsilon$, and thus the log-likelihood ratio statistic is
\[\log(\Lambda) = \sum_{i=1}^n \log\left(\frac{1 - \epsilon}{p_i}\ind\{u_i \in \cI_i\} + \epsilon\right). \] 
This motivates our adaptive test statistic for the ITS method:
\begin{equation}\label{eq:its-stat}
\phi(\xi_{1:n}, \y_{1:n}) = \max_{\epsilon}\sum_{i=1}^n \log\left(\frac{1 - \epsilon}{p_i}\ind\{u_i \in \cI_i\} + \epsilon\right),
\end{equation}
which inherently assigns greater weights to more informative tokens, specifically those with smaller $p_i$.

\begin{remark}
Notice that although our test statistics are derived from Huber's $\epsilon$-contamination model, the problem we address differs from the sparse detection problem extensively studied in the statistics literature \citep{ingster1996some,donoho2004higher,cai2014optimal,arias2019detection}. In their work, the hypothesis testing problems are
\[H_0: X_1, \ldots, X_n \sim F 
\quad \text{v.s.} \quad 
H_a: X_1, \ldots, X_n \sim (1 - \epsilon)F(\cdot) + \epsilon F(\cdot - \mu),\]
whereas in our setting the hypothesis testing problems are
\[H_0: X_i \sim F 
\quad \text{v.s.} \quad 
H_a: X_i \sim (1 - \epsilon)F(\cdot) + \epsilon F_i(\cdot),\]
for $i = 1, \ldots, n$, where $F_i$ denotes a distribution that depends on the index $i$. In fact, our framework is more closely related to the setting in \citet{li2024robust}. However, we develop the model under a different watermark generation framework and propose alternative methods to address the problem. 
\end{remark}

\section{Sub-string identification}\label{sec:cpd}
In this paper, we focus on a relatively underexplored question in the existing literature: Given that the global null hypothesis $H_0$ is rejected, how can we identify the AI-generated sub-strings within the modified text $\y_{1:m}$?

To formalize the setup, we assume that the text published by a third-party user has the following structure:
\begin{equation}\label{eq-seq}
\underbrace{\y_1\y_2\cdots \y_{\tau_1}}_{\text{non-watermarked}}
\underbrace{\y_{\tau_1+1}\cdots \y_{\tau_2}}_{\text{watermarked}}
\underbrace{\y_{\tau_2+1}\cdots\y_{\tau_3}}_{\text{non-watermarked}}
\underbrace{\y_{\tau_3+1}\cdots \y_{\tau_4}}_{\text{watermarked}}\cdots,
\end{equation}
where the sub-strings $\y_{\tau_1+1}\cdots \y_{\tau_2}$ and $\y_{\tau_3+1}\cdots \y_{\tau_4}$ are watermarked. Importantly, the ordering of watermarked and non-watermarked sub-strings is arbitrary and does not impact the proposed method. The primary objective here is to accurately segment the text into watermarked and non-watermarked sub-strings.

We follow the framework proposed by \citet{li2024segmenting} to frame it as a change point detection problem. We define a sequence of moving windows $\mathcal{I}_i=[(i-B/2)\vee 1,(i+B/2)\wedge m]$, where $B$ is the window size, assumed to be an even number for simplicity, and $1 \leq i \leq m$. For each sub-string, we calculate a randomization-based $p$-value as follows:
\begin{align}\label{pvaluecalculation}
\pt_i = \frac{1}{T+1}\left(1+\sum^{T}_{t=1}\ind\{\cM(\xi_{1:n},\y_{\mathcal{I}_i})\leq \cM(\xi_{1:n}^{(t)},\y_{\mathcal{I}_i})\}\right), \quad 1\leq i\leq m,    
\end{align}
where $\cM(\xi_{1:n},\y_{\mathcal{I}_i})=\max_{1\leq k\leq n}\phi(\xi_{\mathcal{J}_k},\y_{\mathcal{I}_i})$, with $\mathcal{J}_k=[(k-B/2)\vee 1,(k+B/2)\wedge n]$.

The text is thereby transformed into a sequence of $p$-values: $\pt_1,\dots,\pt_m$. Under the null hypothesis, the $p$-values are approximately uniformly distributed, whereas under the alternative, they concentrate near zero. In a simplified scenario where the published text is divided into two halves—one watermarked and the other non-watermarked (or vice versa)—the change point location can be identified via:
\begin{align}\label{def-S}
\hat{\tau}=\argmax_{1\leq \tau<m}S_{1:m}(\tau),\quad S_{1:m}(\tau):=\sup_{t\in [0,1]}\frac{\tau(m-\tau)}{m^{3/2}}|F_{1:\tau}(t)-F_{\tau+1:m}(t)|,    
\end{align}
where $F_{a:b}(t)$ represents the empirical cumulative distribution function (cdf) of $\{\pt_i\}_{i=a}^{b}$. Observe that the proposed test statistic \eqref{def-S} is closely related to the scan statistic in the sparse detection literature \citep{arias2019detection}, as both are defined as the supremum of the difference between two empirical distributions. Indeed, the underlying idea originates from classical goodness-of-fit tests in the nonparametric literature, such as the Kolmogorov–Smirnov test.

To determine whether $\hat{\tau}$ is statistically significant, we employ a block bootstrap-based approach. Conditional on $\mathcal{F}_m$, the $p$-value sequence exhibits $B$-dependence, meaning $\pt_i$ and $\pt_j$ are independent only if $|i-j| > B$. Traditional bootstrap or permutation methods for independent data are unsuitable here, as they fail to capture the dependence between neighbouring $p$-values. Instead, we adopt the moving block bootstrap for time series data \citep{kunsch1989jackknife,liu1992moving}.

Given a block size $B'$, we construct $m-B'+1$ blocks as $\{\pt_{i},\dots,\pt_{i+B'-1}\}$ for $1\leq i\leq m-B'+1$. We then randomly sample $m/B'$ blocks (assuming $m/B'$ is an integer) with replacement and concatenate them to form resampled $p$-values $\pt_1^*,\dots,\pt_m^*$. Based on these bootstrapped $p$-values, we compute $F^*_{a:b}(t)$ and define:
\begin{align*}
S^*_{1:m}(\tau)=\sup_{t\in [0,1]}\frac{\tau(m-\tau)}{m^{3/2}}|F_{1:\tau}^*(t)-F_{\tau+1:m}^*(t)|.    
\end{align*}
This procedure is repeated $T'$ times, resulting in statistics $S^{*,(t)}_{1:m}(\tau)$ for $t=1,2,\dots,T'$. The corresponding bootstrap-based $p$-value is:
\begin{align}\label{ptildestatistic}
\tilde{\pt}_{T'}=\frac{1}{T'+1}\left(1+\sum^{T'}_{t=1}\ind\left\{\max_{1\leq \tau<m}S_{1:m}(\tau)\leq \max_{1\leq \tau<m}S^{*,(t)}_{1:m}(\tau)\right\}\right). 
\end{align}
We conclude that a statistically significant change point exists if $\tilde{\pt}_{T'}\leq \alpha$.

Before presenting the details of the algorithm for multiple change-point detection, we emphasize the advantages of our proposed method for substring identification. Compared to the approach of \citet{li2024fastcpd}, our framework employs adaptive test statistics to generate the sequence of $p$-values. These adaptive test statistics are more powerful than the original test statistics, leading to enhanced watermark detection and the generation of higher-quality $p$-value sequences, which in turn facilitate more accurate identification of true change points. While the original method requires estimating NTPs based on prompt estimation—a significantly more challenging task—our framework eliminates the need for prompt estimation. Specifically, we partition the text into small segments, allowing the preceding tokens to serve as natural prompts and provide accurate NTP estimates for each segment. Since the true prompt only affects the initial segment, we simply use an empty set as the initial prompt, thereby eliminating the need for precise prompt estimation.

\subsection{Binary Segmentation}
In this section, we present an algorithm for separating watermarked and non-watermarked sub-strings by identifying multiple change point locations. The literature on identifying multiple change points generally features two main types of algorithms: (i) exact or approximate optimization through minimizing a penalized cost function \citep{harchaoui2010multiple, TRUONG2020107299, doi:10.1080/01621459.2012.737745, li2024fastcpd, zhang2023sequential}, and (ii) approximate segmentation algorithms, for example, binary segmentation \citep{vostrikova1981detecting}, circular binary segmentation \citep{olshen2004circular}, Wild Binary Segmentation(WBS) \citep{10.1214/14-AOS1245}. Our proposed algorithm is based on the widely used binary segmentation method, a top-down approach for detecting multiple change points. 

Initially proposed by \cite{vostrikova1981detecting}, binary segmentation identifies a single change point using a CUSUM-like procedure and subsequently applies a divide-and-conquer approach to detect additional change points within sub-segments until a stopping criterion is met. However, as a greedy algorithm, binary segmentation can be less effective when multiple change points are present. WBS \citep{10.1214/14-AOS1245} and seeded binary segmentation (SeedBS) \citep{10.1093/biomet/asac052} improve upon binary segmentation by considering multiple segments to detect and aggregate potential change points. SeedBS further addresses the issue of excessively long sub-segments in WBS, which may contain several change points. It introduces multiple layers of intervals, each consisting of a fixed number of intervals of varying lengths and shifts, to improve the search for change points. 

When comparing multiple candidates for the next change point, the narrowest-over-threshold (NOT) method prioritizes narrower sub-segments preventing the algorithm from considering sub-segments containing two or more change points \citep{10.1111/rssb.12322}. Based on these ideas, we propose an effective algorithm to identify the change points that separate watermarked and non-watermarked sub-strings. The details of this algorithm are presented in Algorithm~\ref{algorithm:seedbs-not} in Supporting Information B. More discussions about the algorithm are deferred to Supporting Information E.

\section{Numerical experiments}\label{sec:numerical-experiments}
We consider two LLMs: \verb|Meta-Llama-3-8B| from Meta \citep{llama3modelcard} and \verb|Mistral-7B-v0.1| from Mistral AI \citep{jiang2023mistral}, along with two watermarking techniques, EMS and ITS, for watermarked text generation. We select the hyperparameters according to the optimal values recommended in \citet{li2024segmenting} and \citet{li2025likelihood}, setting $B = B' = 20$ throughout our experiments. To study the impact of the window size $B$, we fix $B' = 20$ and vary the value of $B$ in the set $\{10, 20, 30, 40, 50\}$. Table~\ref{tab:diff_B} reports the Rand index values for each setting. A common recommendation in the time series literature is to set $B = C n^{1/3}$, where $n$ denotes the sample size, as shown in Corollary~1 of \citet{lahiri1999theoretical}. Based on our experience, choosing $B = \lfloor 3 n^{1/3} \rfloor$ (e.g., when $n=500$, $B \approx 20$) often yields good finite-sample performance. Further discussion is provided in the Supporting Information D. In the main contents, we present results for the \verb|Meta-Llama-3-8B| model using the EMS method, while results for other settings are included in Supporting Information B and C.

\subsection{Test Statistics Based on Estimated NTPs}\label{sec:est-ntp}
In practice, the true NTPs $\{p_i\}$ are unknown, even when the LLM is accessible since the original prompt is never observed. 
In \citet{li2025likelihood}, the authors proposed to estimate prompt and using the estimated prompt to estimate NTP. However, in our segmenting framework, we partition the text into small segments, where
preceding tokens naturally serve as prompts, providing accurate NTP estimates for each segment. The true prompt affects only
the initial segment, making precise estimation redundant; thus, we can conveniently use an empty set as the initial prompt.

Let $q_i$ be a generic estimate of $p_i$. In practice, the estimated NTPs $q_i$ may be extremely small, leading to instability in the proposed method. To mitigate this, we follow the idea from \citet{li2025likelihood} to regularize the estimated NTPs to enhances the algorithm’s robustness. Specifically, given an estimated NTP $q_i$, we define the regularized NTP as
\[\mathrm{S}(q_i, \lambda) = \lambda q_i + (1-\lambda) p_{i,0},\]
where $\lambda \in (0, 1)$ and $(p_{1,0}, \dots, p_{n,0})$ is a prespecified vector of probabilities. We then modify the test statistic using the regularized NTPs:
\begin{align}
    \phi_{\mathrm{shrinkage}}(\xi_{1:n}, \y_{1:n}; q_{1:n}, \lambda) = \frac{1}{n} \sum_{i=1}^n \frac{1 - \mathrm{S}(q_i, \lambda)}{\mathrm{S}(q_i, \lambda)} \log(\xi_{i,\y_i}). \label{eq:shrinkage-test-statistics}
\end{align}

We compare three methods, all based on the randomization test proposed in Section~\ref{sec:dect}, which differ only in the choice of test statistics. The \texttt{baseline} method uses the original test statistic defined in \eqref{eq-stat1}, which is also used in \citet{li2024segmenting}. The \texttt{oracle} method employs the adaptive test statistic from \eqref{eq:lrt-stat1}. The \texttt{empty} method uses the test statistic in \eqref{eq:shrinkage-test-statistics} with a shrinkage parameter $\lambda = 0.5$ and $q_i = q_{\emptyset,i}$, where $q_{\emptyset,i}$ represents the estimated NTPs with an empty prompt. The \texttt{optim} method uses the same test statistic as the \texttt{empty} approach with $\lambda = 0.5$ and $q_i = q_{\mathrm{opt},i}$, where $q_{\mathrm{opt},i}$ is the estimated NTPs derived from solving the optimization problem using the instruction set $\mathcal{P}_{\mathrm{opt}}$ \citep{li2025likelihood}. The idea of this approach is to first generate an instruction set containing candidate prompts, and then select the ``best prompt'' by maximizing the likelihood. As the instruction set constructed by \citet{li2025likelihood} relies on access to the true prompt—which is not feasible in practical applications—we omit the details here. This method is included primarily to illustrate that, within our framework, it is unnecessary to estimate the true prompt.

Additionally, we adapt several state-of-the-art watermark detection methods, including \citet{xie2025debiasing} and \citet{li2024robust}, to our segmentation framework, and also compare with the approach of \citet{kashtan2023information}. Our proposed method consistently outperforms these alternatives across all settings. Due to space constraints, further implementation details and results are provided in the Supporting Information B.

\subsection{Segmenting Watermarked Texts}\label{sec:segmenting-watermarked-texts}
In this section, followed setting in \citet{li2024segmenting}, we consider two types of attacks—insertion and substitution—that can introduce change points in LLM-generated texts. To evaluate the effectiveness of the proposed method, we examine four settings with varying numbers of change points and segment lengths. In all settings, we perform each experiment with $T = 99$, $T' = 999$ permutations and $100$ independent Monte Carlo replications.

\begin{itemize}
    \item \textbf{Setting 1 (No change points).} Generate $500$ tokens with watermarks. This setting is used to evaluate false positives, and the corresponding results are deferred to the Supporting Information B.
    \item \textbf{Setting 2 (Insertion attack).} Generate $500$ tokens with watermarks, then insert up to $250$ non-watermarked tokens at approximately index $250$ while preserving semantic coherence. This setting introduces a single change point near index $251$.
    \item \textbf{Setting 3 (Substitution attack).} Generate $500$ tokens with watermarks, then substitute tokens with indices ranging from $201$ to $300$ with non-watermarked text following 
    \begin{align}
        \y_i = \begin{cases}
            \Gamma(\xi_i, p(\cdot \mid \y_{-n_0:i-1})) & \text{if } e_i = 0, \\
            \text{Multinomial}(p(\cdot \mid \y_{-n_0:i-1})) & \text{if } e_i = 1.
        \end{cases} \label{eq:substitution-ei}
    \end{align}
    This setting introduces two change points at indices $201$ and $301$.
    \item \textbf{Setting 4 (Insertion and substitution attacks).} Generate $500$ tokens with watermarks, substitute tokens with indices ranging from $101$ to $200$ with non-watermarked text following \eqref{eq:substitution-ei}, and insert up to $100$ non-watermarked tokens at approximately index $300$. This setting introduces four change points at indices $101$, $201$, around $301$, and around $401$.
\end{itemize}

\begin{figure}
    \centering
    \includegraphics[width=0.75\linewidth]{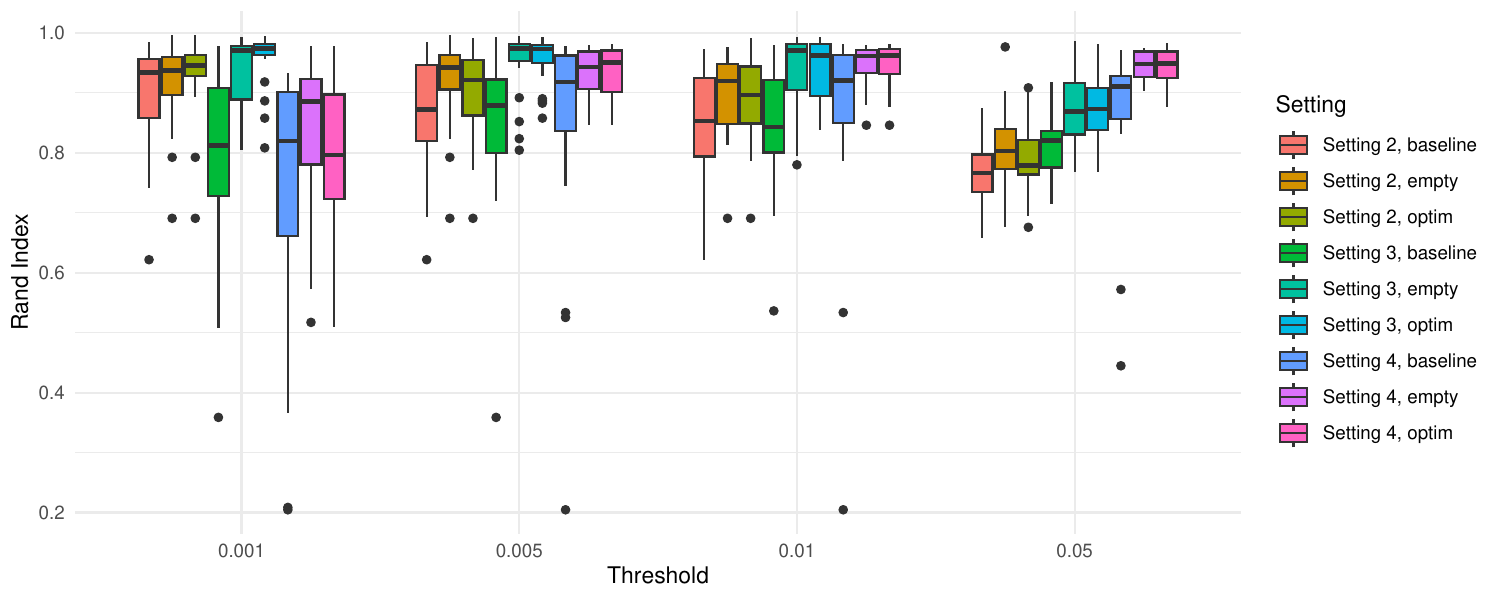}
    \caption{Boxplots of the Rand index comparing clusters identified through detected change points with the true clusters defined by the true change points, for different thresholds in the EMS framework using the \texttt{Llama} LLM.}
    \label{fig:ml3-EMS-rand_index}
\end{figure}

In these experiments, we evaluate clustering performance by comparing the clusters identified via detected change points with the ground truth clusters defined by the true change points, using the Rand index \citep{doi:10.1080/01621459.1971.10482356}. A higher Rand index indicates superior clustering accuracy. As shown in Figure~\ref{fig:ml3-EMS-rand_index}, both adaptive methods consistently outperform the \texttt{baseline} method across all scenarios, further illustrating the advantages of the proposed approach. Moreover, the \texttt{empty} method performs comparably to, and sometimes better than, the \texttt{optim} method. This result supports our claim that prompt estimation is unnecessary in our framework.

Figure~\ref{fig:example-pvalue} further presents the $p$-value sequences produced by different methods under a fixed prompt. Recall that when no watermark is present, $p$-values follow a $\text{Uniform}[0, 1]$ distribution, whereas the presence of a watermark results in $p$-values concentrated near zero. Relative to the \texttt{baseline} method, the two adaptive methods generate higher-quality $p$-value sequences—that is, sequences that more closely align with the expected distribution under each scenario. Since the effectiveness of change point detection critically depends on the quality of the $p$-value sequences, this finding explains why the proposed methods consistently outperform the baseline in segmenting watermarked text. Once again, the \texttt{empty} method performs comparably to the \texttt{optim} method, reinforcing that prompt estimation is not required in our framework.

\begin{figure}
    \centering
    \includegraphics[width=0.75\linewidth]{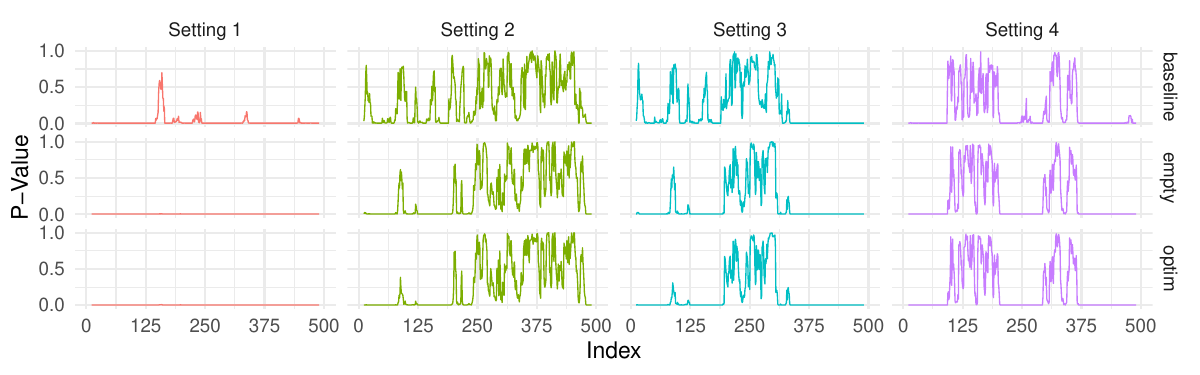}
    \caption{Comparison of $p$-value sequences generated by different methods under a fixed prompt. The two adaptive methods produce higher-quality $p$-value sequences than the \texttt{baseline}, explaining their superior performance in change-point detection. The \texttt{empty} method performs comparably to the \texttt{optim} method, indicating that prompt estimation is unnecessary in our framework.}
    \label{fig:example-pvalue}
\end{figure}

\section{Conclusion}\label{sec:conc}
In this study, we extended the likelihood-based LLM detection method to a more general class of adaptive methods. By formulating the segmentation task as a change-point detection problem, we employed adaptive test statistics to accurately identify watermarked substrings and validated the effectiveness of our approach through extensive simulation studies. These experiments demonstrated the substantial advantages of the proposed methods over existing techniques.

An important direction for future work is to address scenarios in which published texts contain mixed watermarked content generated by different LLMs employing distinct watermarking schemes. This setting introduces the challenge of segmenting text into substrings attributable to different LLMs, each characterized by its own sequence of keys, $p$-values, and change points. The development of algorithms capable of aggregating and analysing such results represents a promising avenue for advancing watermark detection and segmentation research.

\bibliographystyle{plainnat}
\bibliography{main}

\newpage
\appendix
\noindent\textbf{\LARGE Appendices}
\renewcommand\thesection{\Alph{section}}
\numberwithin{equation}{section}
\numberwithin{table}{section}
\numberwithin{figure}{section}

\section{Proofs of the main results}
We provide theoretical proofs for all results to ensure self-containment.
\begin{proof}[Proof of Proposition \ref{thm:error-control}]
~\\
(i) For simplicity of notation, let $\varphi \coloneqq \phi(\xi_{1:n}, \y_{1:m})$ and $\varphi^{(t)} \coloneqq \phi(\xi_{1:n}^{(t)}, \y_{1:m})$ for $t = 1, \dots, T$. Under the null hypothesis, $\xi_{1:n}$ is independent of $\y_{1:m}$, which implies that the pairs 
$$(\xi_{1:n}, \y_{1:m}), (\xi_{1:n}^{(1)}, \y_{1:m}), \dots, (\xi_{1:n}^{(T)}, \y_{1:m})$$
follow the same distribution. Therefore, $\varphi, \varphi^{(1)}, \dots, \varphi^{(T)}$ are exchangeable. This exchangeability ensures that the rank of $\varphi$ relative to $\{\varphi, \varphi^{(1)}, \dots, \varphi^{(T)}\}$ is uniformly distributed. Let the order statistics be denoted as $\varphi_{(1)} \leq \cdots \leq \varphi_{(T + 1)}$. Then, we have
\[P\left(\varphi = \varphi_{(j)}\right) = \frac{1}{T + 1},\qquad j = 1, \dots, T + 1.\]
Thus, for $j = 1, \dots, T + 1$, we obtain
\[P\left(\pt_T \leq \frac{j}{T + 1}\right) = P\left(\varphi \in \left\{\varphi_{(T+2-j)}, \dots, \varphi_{(T+1)}\right\}\right) = \frac{j}{T + 1}.\]
Finally, we have
\[P\left(\pt_T \leq \alpha \right) = \frac{\lfloor (T+1)\alpha \rfloor}{T + 1} \leq \alpha.\]

(ii) By Chebyshev's inequality we have
\[P\left( |\phi(\xi'_{1:n}, \y_{1:m}) - \sE_{\xi'}| \geq \epsilon \sd_{\xi'} \lvert \cF_m \right) \leq \frac{\Var(\phi(\xi'_{1:n}, \y_{1:m}) \lvert \cF_m)}{\epsilon^2 \sd_{\xi'}^2} = \frac{1}{\epsilon^2},\]
for all $\epsilon \geq 0$. Similarly, we have $\phi(\xi'_{1:n}, \y_{1:m}) = \sE_{\xi'} + O_p(\sd_{\xi'})$ given $\cF_m$.

Let the distribution of $\phi(\xi'_{1:n}, \y_{1:m})$ conditional on $\cF_m$ be denoted by $F$, and the empirical distribution of $\{\varphi^{(t)}\}_{t=0}^T$ by $F_T$, where we set $\varphi^{(0)} = \varphi$. 
Let $q_{1-\alpha,T} = \varphi_{(T+2-j_\alpha)}$ with $j_\alpha = \lfloor (T+1)\alpha \rfloor$. Note that $F_T(q_{1-\alpha,T}) = 1 - (j_\alpha - 1) / (T+1)$. Our test rejects the null whenever $\pt_T \leq \alpha$, which is equivalent to rejecting the null if $\varphi \geq \varphi_{(T+2-j_\alpha)}$.
By the Dvoretzky–Kiefer–Wolfowitz inequality, we have
\begin{align*}
&P(|F_T(q_{1-\alpha,T}) - F(q_{1-\alpha,T})| > \epsilon \lvert \cF_m)
\leq P\left(\sup_x |F_T(x) - F(x)| > \epsilon \lvert \cF_m\right) \leq C_1 \exp(-2T\epsilon^2)
\end{align*}
for some constant $C_1 > 0$, which implies that, with probability greater than $1 - C_1 \exp(-2T\epsilon^2)$, 
$F(q_{1-\alpha,T}) < 1 - (j_\alpha - 1) / (T+1) + 2\epsilon$.
Define $F^{-1}(t) = \inf\{s: F(s) \geq t\}$ and the event $\mathcal{A}_T = \{q_{1-\alpha,T} < F^{-1}(1 - (j_\alpha - 1) / (T+1) + 2\epsilon)\}$. Then we have $P(\mathcal{A}_T \lvert \cF_m) \geq 1 - C_1 \exp(-2T\epsilon^2)$. In addition, as $\phi(\xi'_{1:n}, \y_{1:m}) = \sE_{\xi'} + O_p(\sd_{\xi'})$, we have $F^{-1}(s) = \sE_{\xi'} + O(\sd_{\xi'})$ for any $s < 1$.

Notice that $\sd_{\xi}^{-1} \bigl(\phi(\xi_{1:n}, \y_{1:m}) - \sE_{\xi}) \lvert \cF_m]\bigr) = O_p(1)$. Hence, for $T > 2/\epsilon - 1$,
\begin{align*}
&P(\phi(\xi_{1:n}, \y_{1:m}) \geq q_{1 - \alpha, T} \lvert \cF_m) \\
\geq &P\left( \sd_{\xi}^{-1} \bigl( \phi(\xi_{1:n}, \y_{1:m}) - \sE_{\xi} \bigr) + \sd_{\xi}^{-1} \sE_{\xi} \geq \sd_{\xi}^{-1} q_{1 - \alpha, T}, \mathcal{A}_T \lvert \cF_m \right) \\
\geq &P\Big( O_p(1) + \sd_{\xi}^{-1} \sE_{\xi} \geq \sd_{\xi}^{-1} F^{-1}(1 - (j_\alpha - 1) / (T+1) + 2\epsilon), \mathcal{A}_T \lvert \cF_m \Big) \\
\geq &P\Big( O_p(1) + \sd_{\xi}^{-1} \sE_{\xi} \geq \sd_{\xi}^{-1} F^{-1}(1 - \alpha + 3\epsilon), \mathcal{A}_T \lvert \cF_m \Big) \\
\geq &P\Big( O_p(1) + \sd_{\xi}^{-1} \sE_{\xi} \geq \sd_{\xi}^{-1} \bigl(\sE_{\xi'} + O(\sd_{\xi'})\bigr), \mathcal{A}_T \lvert \cF_m \Big) \\
\geq &1 - C_1 \exp(-2T\epsilon^2) + o(1),
\end{align*}
where we have used Condition $\sd_{\xi} = o(\sE_{\xi})$, $\sd_{\xi'} = o(\sE_{\xi})$ and $\sE_{\xi'} = o(\sE_{\xi})$to establish the convergence.
\end{proof}

\begin{proof}[Proof of Proposition \ref{prop-max}]
We only prove the result under the setting of Example \ref{ex2} as the proof for Example \ref{ex1} is similar with the help of Bernstein's inequality. Because $|h_i|\leq 1/4$, by Hoeffding's inequality, we have
\begin{align*}
&P\left(|\phi(\xi_{a:a+B-1},\y_{b:b+B-1})-\mathbb{E}[\phi(\xi_{a:a+B-1},\y_{b:b+B-1})|\y_{1:m},y_{-n_0:n}]|>t|\y_{1:m},y_{-n_0:n}\right)
\\ \leq&  2\exp\left(-8Bt^2\right). 
\end{align*}
By the union bound, we have 
\begin{align*}
&P\Bigg(\max_{1\leq a\leq n-B+1,1\leq b\leq m-B+1}|\phi(\xi_{a:a+B-1},\y_{b:b+B-1})
\\&-\mathbb{E}[\phi(\xi_{a:a+B-1},\y_{b:b+B-1})|\y_{1:m},y_{-n_0:n}]|
>t|\y_{1:m},y_{-n_0:n}\Bigg)
\\ \leq&  2(n-B+1)(m-B+1)\exp\left(-8Bt^2\right).  
\end{align*}
Integrating out the strings in $y_{1:n}$ that are not contained in $\y_{1:m}$, we obtain
\begin{align*}
&P\Bigg(\max_{1\leq a\leq n-B+1,1\leq b\leq m-B+1}|\phi(\xi_{a:a+B-1},\y_{b:b+B-1})
\\&-\mathbb{E}[\phi(\xi_{a:a+B-1},\y_{b:b+B-1})|\y_{1:m},y_{-n_0:n}]|>t|\cF_m\Bigg)
\\ \leq&  2(n-B+1)(m-B+1)\exp\left(-8Bt^2\right),
\end{align*}
where $\cF_m = [\y_{1:m},y_{-n_0:0}]$. Thus, conditional on $\cF_m$, we have
\[\max_{a,b}\{|\mathbb{E}[\phi(\xi_{a:a+B-1},\y_{b:b+B-1})| \cF_m]-\phi(\xi_{a:a+B-1},\y_{b:b+B-1})|\}=O(C_{N,B}).\]
Note that 
\begin{align*}
\phi(\xi_{1:n},\y_{1:m})=&\max_{a,b}\phi(\xi_{a:a+B-1},\y_{b:b+B-1})
\\ \geq& \max_{a,b}\mathbb{E}[\phi(\xi_{a:a+B-1},\y_{b:b+B-1})|\cF_m] 
\\&- \max_{a,b}\{\mathbb{E}[\phi(\xi_{a:a+B-1},\y_{b:b+B-1})|\cF_m]-\phi(\xi_{a:a+B-1},\y_{b:b+B-1})\}
\\ =& \max_{a,b}\mathbb{E}[\phi(\xi_{a:a+B-1},\y_{b:b+B-1})|\cF_m] 
+O(C_{N,B}).
\end{align*}
On the other hand, for a randomly generated key $\xi'_{1:n}$, $\mathbb{E}[\phi(\xi_{a:a+B-1}',\y_{b:b+B-1})|\cF_m]=0$ for all $a,b$. By the same argument, we get
\begin{align*}
P\left(\phi(\xi_{1:n}',\y_{1:m})>t|\cF_m\right)=&P\left(\max_{1\leq a\leq n-B+1,1\leq b\leq m-B+1}\phi(\xi_{a:a+B-1}',\y_{b:b+B-1})>t\Big|\cF_m\right)
\\ \leq&  2(n-B+1)(m-B+1)\exp\left(-8Bt^2\right),    
\end{align*}
which suggests that $F^{-1}(s)=O(C_{N,B})$ with $F$ being 
the distribution of $\phi(\xi_{1:n}', \y_{1:m})$ conditional on $\cF_m$ and $F^{-1}(t)=\inf\{s:F(s)\geq t\}$. The rest of the arguments are similar to those in the proof of Theorem \ref{thm:error-control}. We skip the details.
\end{proof}

\begin{proof}[Proof of Theorem~\ref{thm:its}]
Recall in ITS setting that the test statistic is given by:
\[\phi(\xi_{1:n}, y_{1:n}) = \frac{1}{n}\sum_{i=1}^n w_i (u_i - 1 / 2)\left(\frac{\pi_i(y_i) - 1}{V - 1} - \frac{1}{2}\right) \coloneqq \frac{1}{n}\sum_{i = 1}^n w_i h_i(\xi_i, y_i).\]
Because $|h_i|\leq 1/4$, by Hoeffding's inequality, we have
\begin{align*}
&P\left(|\phi(\xi_{a:a+B-1},\y_{b:b+B-1})-\mathbb{E}[\phi(\xi_{a:a+B-1},\y_{b:b+B-1})|\y_{1:m},y_{-n_0:n}]|>t|\y_{1:m},y_{-n_0:n}\right)
\\ \leq&  2\exp\left(-8Bt^2 / \Omega_{\max}^2\right). 
\end{align*}
By the union bound, we have 
\begin{align*}
&P\Bigg(\max_{1\leq a\leq n-B+1,1\leq b\leq m-B+1}|\phi(\xi_{a:a+B-1},\y_{b:b+B-1})
\\&-\mathbb{E}[\phi(\xi_{a:a+B-1},\y_{b:b+B-1})|\y_{1:m},y_{-n_0:n}]|
>t|\y_{1:m},y_{-n_0:n}\Bigg)
\\ \leq&  2(n-B+1)(m-B+1)\exp\left(-8Bt^2 / \Omega_{\max}^2\right).  
\end{align*}
Integrating out the strings in $y_{1:n}$ that are not contained in $\y_{1:m}$, we obtain
\begin{align*}
&P\Bigg(\max_{1\leq a\leq n-B+1,1\leq b\leq m-B+1}|\phi(\xi_{a:a+B-1},\y_{b:b+B-1})
\\&-\mathbb{E}[\phi(\xi_{a:a+B-1},\y_{b:b+B-1})|\y_{1:m},y_{-n_0:n}]|>t|\cF_m\Bigg)
\\ \leq&  2(n-B+1)(m-B+1)\exp\left(-8Bt^2 / \Omega_{\max}^2\right),
\end{align*}
where $\cF_m = [\y_{1:m},y_{-n_0:0}]$. Thus, conditional on $\cF_m$, we have
\[\max_{a,b}\{|\mathbb{E}[\phi(\xi_{a:a+B-1},\y_{b:b+B-1})| \cF_m]-\phi(\xi_{a:a+B-1},\y_{b:b+B-1})|\}=O(C_{N,B}).\]
Note that 
\begin{align*}
\phi(\xi_{1:n},\y_{1:m})=&\max_{a,b}\phi(\xi_{a:a+B-1},\y_{b:b+B-1})
\\ \geq& \max_{a,b}\mathbb{E}[\phi(\xi_{a:a+B-1},\y_{b:b+B-1})|\cF_m] 
\\&- \max_{a,b}\{\mathbb{E}[\phi(\xi_{a:a+B-1},\y_{b:b+B-1})|\cF_m]-\phi(\xi_{a:a+B-1},\y_{b:b+B-1})\}
\\ =& \max_{a,b}\mathbb{E}[\phi(\xi_{a:a+B-1},\y_{b:b+B-1})|\cF_m] 
+O(C_{N,B}).
\end{align*}
On the other hand, for a randomly generated key $\xi'_{1:n}$, $\mathbb{E}[\phi(\xi_{a:a+B-1}',\y_{b:b+B-1})|\cF_m]=0$ for all $a,b$. By the same argument, we get
\begin{align*}
P\left(\phi(\xi_{1:n}',\y_{1:m})>t|\cF_m\right)=&P\left(\max_{1\leq a\leq n-B+1,1\leq b\leq m-B+1}\phi(\xi_{a:a+B-1}',\y_{b:b+B-1})>t\Big|\cF_m\right)
\\ \leq&  2(n-B+1)(m-B+1)\exp\left(-8Bt^2 / \Omega_{\max}^2\right),    
\end{align*}
which suggests that $F^{-1}(s)=O(C_{N,B})$ with $F$ being 
the distribution of $\phi(\xi_{1:n}', \y_{1:m})$ conditional on $\cF_m$ and $F^{-1}(t)=\inf\{s:F(s)\geq t\}$. The rest of the arguments are similar to those in the proof of Theorem \ref{thm:error-control}. We skip the details.
\end{proof}

\begin{proof}[Proof of Corollary \ref{cor:thm-exp}]
~\\
(i) In Example~\ref{ex1}, the test statistic is given by
\[\phi(\xi_{1:n}, y_{1:n}) = \frac{1}{n}\sum_{i=1}^n\left\{\log(\xi_{i, y_i}) + 1\right\}.\]
Observe that $E_{ik} \coloneqq -\log(\xi_{ik}) / p(k|y_{-n_0:i-1})\sim\text{Exp}(p(k|y_{-n_0:i-1}))$. Since $\xi'_{1:n}$ is independent of $y_{1:n}$,  we have
$-\log(\xi'_{i, y_i})\lvert y_{-n_0:n}\sim\text{Exp}(1)$. Hence, conditional on $y_{-n_0:n}$, we have
\[\bbE[\phi(\xi'_{1:n}, y_{1:n})|y_{-n_0:n}] = 0, \quad \Var(\phi(\xi'_{1:n}, y_{1:n})|y_{-n_0:n}) = \frac{1}{n}.\]

Given $y_{-n_0:n}$, we know $E_{i, y_i} = \min_{1\leq k \leq V}E_{ik}$, which implies $-\log(\xi_{i, y_i}) / p(y_i|y_{-n_0:i-1}) \lvert y_{-n_0:n} \sim \text{Exp}(1)$. It is worth noting that
\begin{align*}
P(-\log(\xi_{i, y_i}) \geq t) = p\left(-\frac{\log(\xi_{i, y_i})}{p(y_i|y_{-n_0:i-1})} \geq \frac{t}{p(y_i|y_{-n_0:i-1})}\right) =\exp\left(-\frac{t}{p(y_i|y_{-n_0:i-1})}\right).
\end{align*}
That is  $-\log(\xi_{i, y_i})\lvert y_{-n_0:n}\sim\text{Exp}(1/p(y_i|y_{-n_0:i-1}))$. Since $\xi_i$s are conditionally independent given $y_{-n_0:n}$. Thus, we have
\begin{align*}
&\bbE[\phi(\xi_{1:n}, y_{1:n})|y_{-n_0:n}] = \frac{1}{n}\sum_{i=1}^n \left(1 - p(y_i|y_{-n_0:i-1}) \right),\\ 
&\Var(\phi(\xi_{1:n}, y_{1:n})|y_{-n_0:n}) = \frac{1}{n^2}\sum_{i=1}^n p(y_i|y_{-n_0:i-1})^2 \leq \frac{1}{n}.
\end{align*}
Therefore, conditions of Theorem~\ref{thm:error-control} are satisfied.

(ii) Recall in Example~\ref{ex2} that the test statistic is given by:
\[\phi(\xi_{1:n}, y_{1:n}) = \frac{1}{n}\sum_{i=1}^n(u_i - 1 / 2)\left(\frac{\pi_i(y_i) - 1}{V - 1} - \frac{1}{2}\right) \coloneqq \frac{1}{n}\sum_{i = 1}^n h_i(\xi_i, y_i),\]
where $\xi_i = (u_i, \pi_i)$. Notice that $h_i(\xi_i, y_i)$ is bounded. Since $\xi_{1:n}'$ is independent of $y_{-n_0:n}$, we have
$u_i'|y_{1:n}\sim\text{Unif}[0, 1]$. Thus, $\bbE[\phi(\xi'_{1:n}, y_{1:n})\lvert y_{1:n}] = 0$.

Denote $\mu_i(\cdot) = p(\cdot \lvert y_{-n_0:i-1})$. Conditional on $y_{-n_0:i}$ and $\pi_i(y_i)$, we know that $u_i$ follows the uniform distribution over the interval 
$[\mu_i(y: \pi_i(y) < \pi_i(y_i)), \mu_i(y: \pi_i(y) \leq \pi_i(y_i))]$. As a result, we can calculate the expected value of $u_i$ given $y_{-n_0:i}$ and $\pi_i(y_i)$ as
\begin{align*}
\bbE[u_i\lvert y_{-n_0:i}, \pi_i(y_i)] &= \frac{1}{2}\left\{\mu_i(y:\pi_i(y) < \pi_i(y_i)) + \mu_i(y:\pi_i(y) \leq \pi_i(y_i))\right\}\\
&= \frac{\mu_i(y_i)}{2} + \mu_i(y:\pi_i(y) < \pi_i(y_i))\\
&= \frac{\mu_i(y_i)}{2} + \frac{\pi_i(y_i) - 1}{V - 1}(1 - \mu_i(y_i))\\
&= \frac{1}{2} + (1 - \mu_i(y_i))\left(\frac{\pi_i(y_i) - 1}{V - 1} - \frac{1}{2}\right),
\end{align*}
where the third equality is because given $\pi_i(y_i) = k$, $\pi_i$ follows the uniform distribution over the permutation space with the restriction $\pi_i(y_i) = k$. Then, we have
\begin{align*}
&\bbE\left[(u_i - 1/2)\left(\frac{\pi_i(y_i) - 1}{V - 1} - \frac{1}{2}\right)\lvert y_{-n_0:i}\right] \\
=& \bbE\left[\bbE\left[(u_i - 1/2)\left(\frac{\pi_i(y_i) - 1}{V - 1} - \frac{1}{2}\right)\lvert y_{-n_0:i}, \pi_i(y_i)\right]\lvert y_{-n_0:i}\right]\\
=& \bbE\left[(1 - \mu_i(y_i))\left(\frac{\pi_i(y_i) - 1}{V - 1} - \frac{1}{2}\right)^2\right]\\
=& \bigl(1 - p(y_i|y_{-n_0:i - 1})\bigr)\bbE\left[\left(\frac{\pi_i(y_i) - 1}{V - 1} - \frac{1}{2}\right)^2\lvert y_{-n_0:i}\right].
\end{align*}

Given $y_i$, $(\pi_i(y_i)-1)/(V-1)$ follows the uniform distribution over the discrete space $\{0, 1 / (V - 1), \dots, 1\}$. Thus we have $\bbE[(\pi_i(y_i) - 1) / (V - 1)] = 1 / 2$ and $\Var((\pi_i(y_i) - 1) / (V - 1)) = C$, which is a constant less than $1 / 12$ (i.e., the variance of $\text{Unif}[0, 1]$), which further implies that $\bbE[\phi(\xi_{1:n}, y_{1:n})] = C n^{-1}\sum_{i=1}^n\bigl(1 - p(y_i\lvert y_{1:i - 1})\bigr)$. Hence, conditions of Theorem~\ref{thm:error-control} are satisfied due to \eqref{eq:cond-examp}.
\end{proof}

\begin{proof}[Proof of Corollary~\ref{cor:its}]
Recall from Example~\ref{ex2} that the test statistic is given by:
\[\phi(\xi_{1:n}, y_{1:n}) = \frac{1}{n}\sum_{i=1}^n w_i (u_i - 1 / 2)\left(\frac{\pi_i(y_i) - 1}{V - 1} - \frac{1}{2}\right) \coloneqq \frac{1}{n}\sum_{i = 1}^n w_i h_i(\xi_i, y_i),\]
where $\xi_i = (u_i, \pi_i)$. First, note that $h_i(\xi_i, y_i)$ is bounded, with a minimum value of $-1/4$ and a maximum value of $1/4$. Consequently, $\sd_{\xi} \leq \frac{1}{4n}\sqrt{\sum_{i=1}^n w_i^2}$ and $\sd_{\xi'} \leq \frac{1}{4n}\sqrt{\sum_{i=1}^n w_i^2}$. Since $\xi_{1:n}'$ is independent of $y_{-n_0:n}$, it follows that $u_i' \mid y_{1:n} \sim \text{Unif}[0, 1]$, and hence $\sE_{\xi'} = 0$.

Define $\mu_i(\cdot) = p(\cdot \mid y_{-n_0:i-1})$. Conditional on $y_{-n_0:i}$ and $\pi_i(y_i)$, $u_i$ is uniformly distributed over the interval 
$[\mu_i(y: \pi_i(y) < \pi_i(y_i)), \mu_i(y: \pi_i(y) \leq \pi_i(y_i))]$. Therefore, the expected value of $u_i$ given $y_{-n_0:i}$ and $\pi_i(y_i)$ is:
\begin{align*}
\bbE[u_i \mid y_{-n_0:i}, \pi_i(y_i)] &= \frac{1}{2}\left\{\mu_i(y:\pi_i(y) < \pi_i(y_i)) + \mu_i(y:\pi_i(y) \leq \pi_i(y_i))\right\} \\
&= \frac{\mu_i(y_i)}{2} + \mu_i(y:\pi_i(y) < \pi_i(y_i)) \\
&= \frac{\mu_i(y_i)}{2} + \frac{\pi_i(y_i) - 1}{V - 1}(1 - \mu_i(y_i)) \\
&= \frac{1}{2} + (1 - \mu_i(y_i))\left(\frac{\pi_i(y_i) - 1}{V - 1} - \frac{1}{2}\right),
\end{align*}
where the third equality holds because, given $\pi_i(y_i) = k$, $\pi_i$ follows a uniform distribution over the permutation space with the restriction $\pi_i(y_i) = k$. Thus, we have:
\begin{align*}
&\bbE\left[w_i (u_i - 1/2)\left(\frac{\pi_i(y_i) - 1}{V - 1} - \frac{1}{2}\right) \mid y_{-n_0:i}\right] \\
=& \bbE\left[w_i \bbE\left[(u_i - 1/2)\left(\frac{\pi_i(y_i) - 1}{V - 1} - \frac{1}{2}\right) \mid y_{-n_0:i}, \pi_i(y_i)\right] \mid y_{-n_0:i}\right] \\
=& \bbE\left[w_i (1 - \mu_i(y_i))\left(\frac{\pi_i(y_i) - 1}{V - 1} - \frac{1}{2}\right)^2 \lvert y_{-n_0:i}\right] \\
=& w_i \bigl(1 - p(y_i \mid y_{-n_0:i - 1})\bigr)\bbE\left[\left(\frac{\pi_i(y_i) - 1}{V - 1} - \frac{1}{2}\right)^2 \mid y_{-n_0:i}\right].
\end{align*}

Given $y_i$, $(\pi_i(y_i) - 1) / (V - 1)$ follows a uniform distribution over the discrete space $\{0, 1 / (V - 1), \dots, 1\}$. Thus, $\bbE[(\pi_i(y_i) - 1) / (V - 1)] = 1 / 2$ and $\Var((\pi_i(y_i) - 1) / (V - 1)) = C$, where $C$ is a constant less than $1 / 12$ (the variance of $\text{Unif}[0, 1]$). Therefore, we have:
\[\sE_{\xi} = C n^{-1}\sum_{i=1}^n w_i \bigl(1 - p(y_i \mid y_{1:i - 1})\bigr).\]
Thus, the conditions of Theorem~\ref{thm:error-control} are satisfied due to \eqref{eq:cond-genestat}.
\end{proof}

\begin{proof}[Proof of Lemma~\ref{lemma:its}]
If $y_i$ is not generated from $\xi_i$, then $u_i$ is independent of $y_i$, and $u_i \sim \text{Unif}[0, 1]$. On the other hand, if $y_i$ is generated from $\xi_i$, then $u_i \in \mathcal{I}_i$, implying that $u_i \sim \text{Unif}[\mathcal{I}_i]$.
\end{proof}

\newpage

\section{Additional numerical results for the EMS method}\label{sec:additional-numerical-results}
Algorithm~\ref{algorithm:seedbs-not} is the SeedBS-NOT algorithm for change point detection in partially watermarked texts.
\begin{algorithm}[tbh]
\linespread{1.0}\selectfont
  \caption{SeedBS-NOT for change point detection in partially watermarked texts}
  \label{algorithm:seedbs-not}
  \begin{algorithmic}
  \Require Sequence of $p$-values $\{p_i\}_{i=1}^m$, decay parameter $a \in [1/2, 1)$ for SeedBS, threshold $\zeta$ for NOT.
  \Ensure Locations of the change points.
  \State Define $I_1 = (0, m]$.  \Comment{Initialization for SeedBS}
  \For{$k \gets 2, \ldots, \lceil \log_{1/a}m \rceil$}
    \State Number of intervals in the $k$-th layer: $n_k = 2 \lceil (1/a)^{k-1} \rceil - 1$.
    \State Length of intervals in the $k$-th layer: $l_k = m a^{k - 1}$.
    \State Shift of intervals in the $k$-th layer: $s_k = (m - l_k) / (n_k - 1)$.
    \State $k$-th layer intervals: $\mathcal{I}_k = \bigcup_{i=1}^{n_k}\{(\lfloor (i-1)s_k \rfloor, \lceil (i-1)s_k + l_k \rceil]\}$.
  \EndFor
  \State Define all seeded intervals $\mathcal{I} = \bigcup_{k = 1}^{\lceil \log_{1/a}m \rceil}\mathcal{I}_k$.  \Comment{End of SeedBS}
  \For{$i \gets 1, \ldots, \lvert \mathcal{I} \rvert$}  \Comment{Start of NOT}
    \State Define the $i$-th interval $I_i = (r_i, s_i]$.
    \State Define $S_{r_i+1:s_i}(\tau):=\sup_{t\in [0,1]}\frac{(\tau - r_i)(s_i - \tau)}{(s_i-r_i)^{3/2}}|F_{r_i+1:\tau}(t)-F_{\tau+1:s_i}(t)|.$
    \State Let $\hat{\tau}_i = \argmax_{r_i < \tau \le s_i} S_{r_i+1:s_i}(\tau)$.
    \State Obtain $\tilde{p}_i$ using block bootstrap as defined in \eqref{ptildestatistic}.
  \EndFor
  \State Define the set of potential change point locations $\mathcal{O} = \{i: \tilde{p}_i < \zeta\}$ and initialize the final set of change points as $\mathcal{S} = \emptyset$.
  \While{$\mathcal{O} \not= \emptyset$}
    \State Select $i = \argmin_{i = 1, \ldots, \lvert \mathcal{O} \rvert} \{ \lvert I_i \rvert \} = \argmin_{i = 1, \ldots, \lvert \mathcal{O} \rvert} \{s_i - r_i\}$.
    \State Update $\mathcal{S} \gets \mathcal{S} \cup \{ \hat{\tau}_i \}; \quad \mathcal{O} \gets \{ j \le \lvert \mathcal{O} \rvert: \hat{\tau}_i \not\in I_j \}$.
  \EndWhile
  \State \Return $\mathcal{S}$.  \Comment{End of NOT}
  \end{algorithmic}
\end{algorithm}

\subsection{Additional Numerical Results for the EMS Method Using the \texttt{Llama} Model}
Figure~\ref{fig:ml3-EMS-pvalue} shows the $p$-value sequence obtained for the first 10 prompts using the \texttt{Llama} model with watermarks generated by the EMS method.

\begin{figure}
    \centering
    \includegraphics[width=\linewidth]{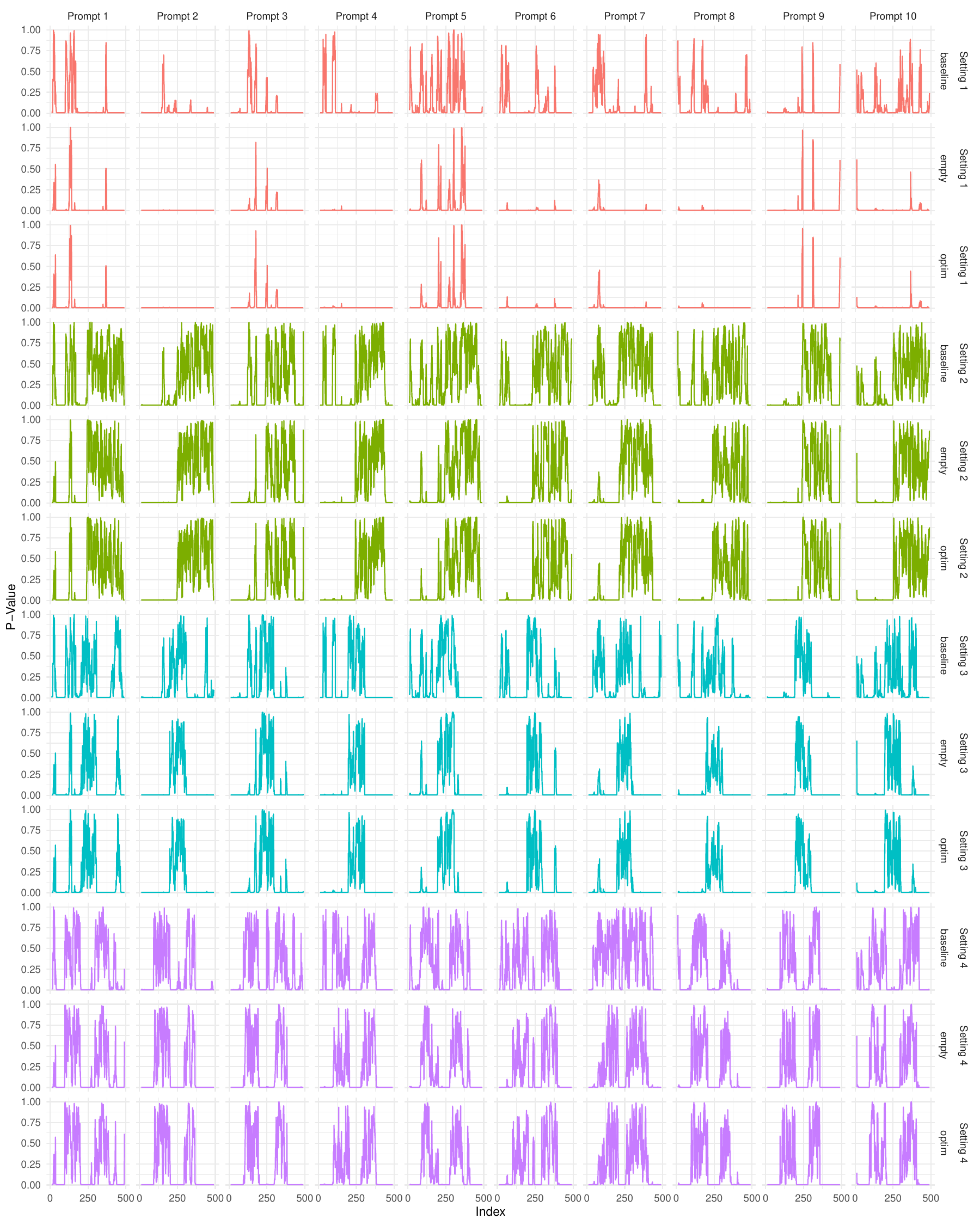}
    \caption{$p$-value sequence calculated using watermarked texts generated from the \texttt{Llama} LLM with the EMS method.}
    \label{fig:ml3-EMS-pvalue}
\end{figure}

\subsection{Additional Numerical Results for the EMS Method Using the \texttt{Mistral} Model}
Figure~\ref{fig:mt7-EMS-rand_index} shows the Rand index for different thresholds in the EMS framework using the \texttt{Mistral} LLM. Finally, Figure~\ref{fig:mt7-EMS-pvalue} presents the $p$-value sequences obtained for the first 10 prompts using the \texttt{Llama} model with watermarks generated by the EMS method.

\begin{figure}
    \centering
    \includegraphics[width=\linewidth]{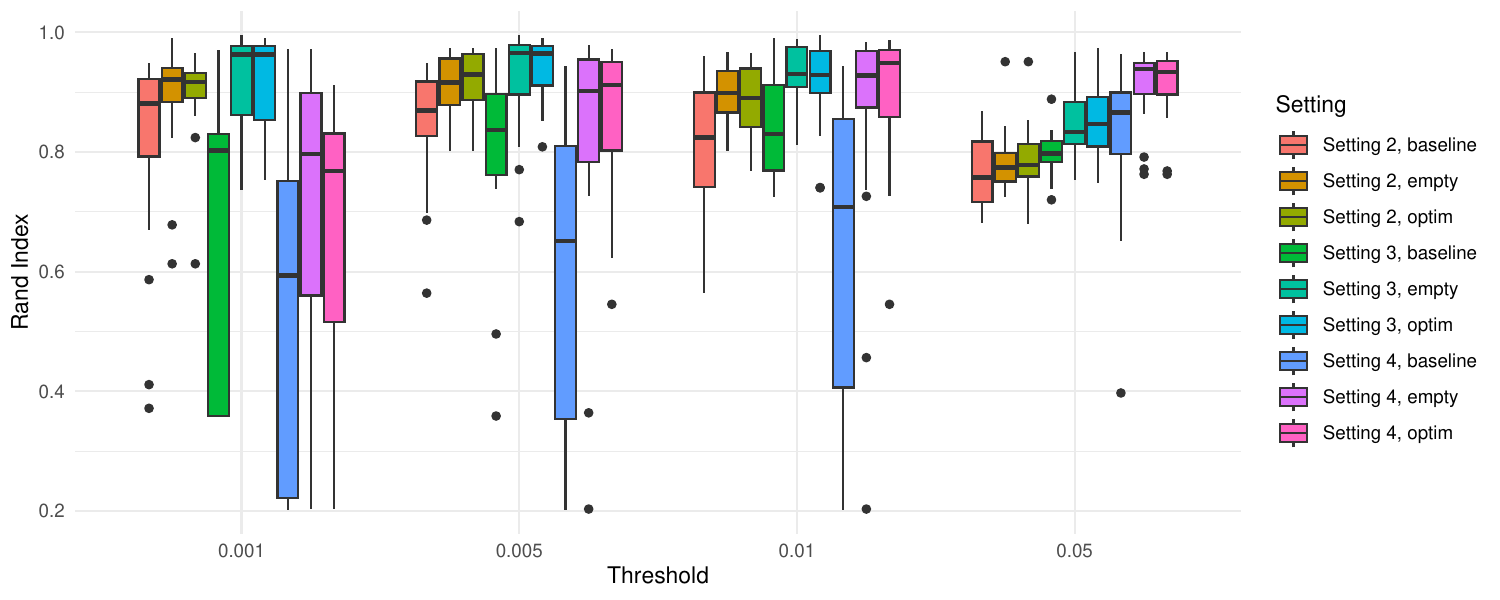}
    \caption{Boxplots of the Rand index comparing clusters identified through detected change points with the true clusters defined by the true change points, for different thresholds in the EMS framework using the \texttt{Mistral} LLM.} \label{fig:mt7-EMS-rand_index}
\end{figure}

\begin{figure}
    \centering
    \includegraphics[width=\linewidth]{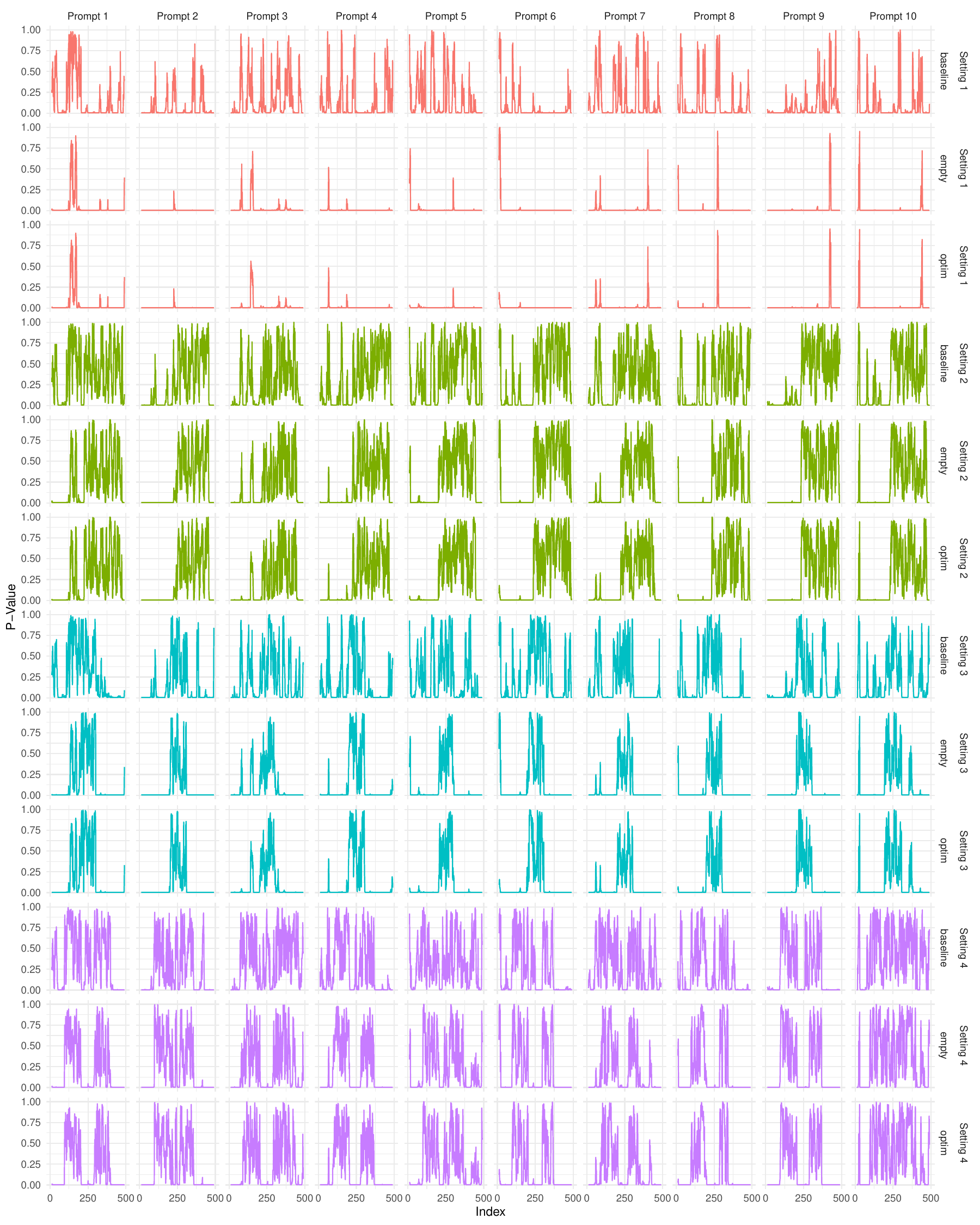}
    \caption{$p$-value sequence calculated using watermarked texts generated from the \texttt{Mistral} LLM with the EMS method.}
    \label{fig:mt7-EMS-pvalue}
\end{figure}
In all settings, the proposed adaptive method outperforms the \texttt{baseline} method, demonstrating the advantages of the proposed approach.

\subsection{Illustration of False Detections}
Figure~\ref{fig:ml3-EMS-false_positives} illustrates the number of false detections for different thresholds. In all settings, the \texttt{oracle} method consistently achieves the lowest false positive rate, followed by the \texttt{optim} method. The \texttt{empty} method performs similarly to the \texttt{optim} method, and all adaptive methods outperform the \texttt{baseline} method. 

\begin{figure}
    \centering
    \includegraphics[width=0.8\linewidth]{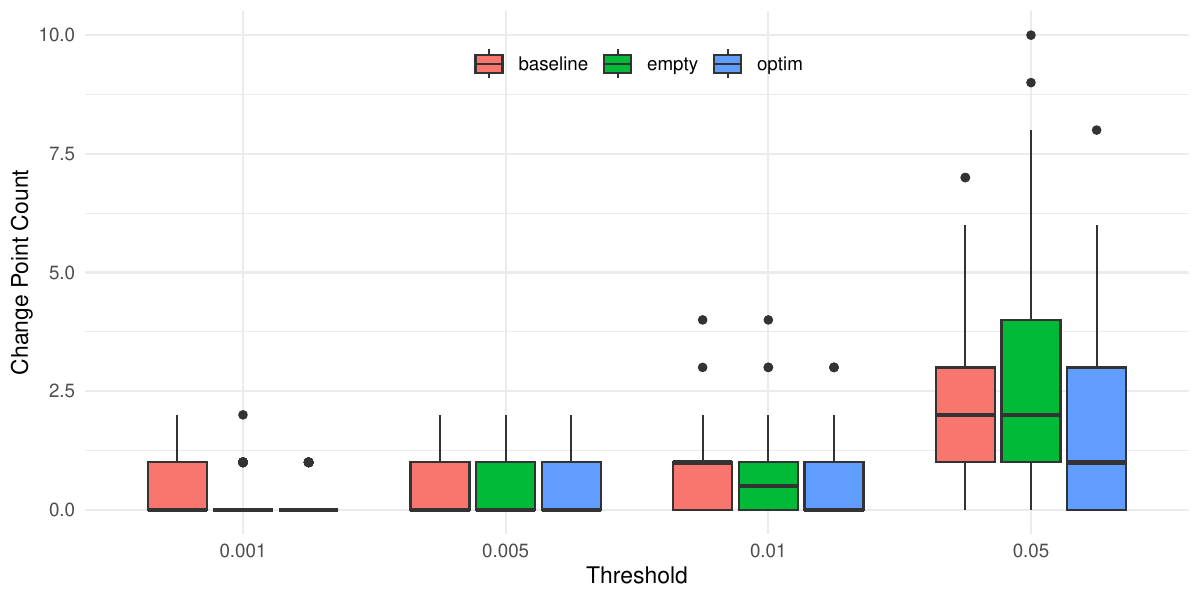}
    \caption{Boxplots of the number of false detections for different thresholds in the EMS framework using the \texttt{Llama} LLM.}
    \label{fig:ml3-EMS-false_positives}
\end{figure}

\subsection{Comparison with Other Test Statistics}
In \citet{kuditipudi2023robust}, the authors proposed using $-\log(1 - \xi_{i, \y_i})$ instead of the test statistic in~\eqref{eq:lrt-stat1}, showing improved performance in watermark detection. We present results based on this statistic in Figures~\ref{fig:ml3-EMS-rand_index-new} and~\ref{fig:mt7-EMS-rand_index-new}, but in our sub-string identification problem, we observed no significant performance differences between the two test statistics.

\begin{figure}
    \centering
    \includegraphics[width=\linewidth]{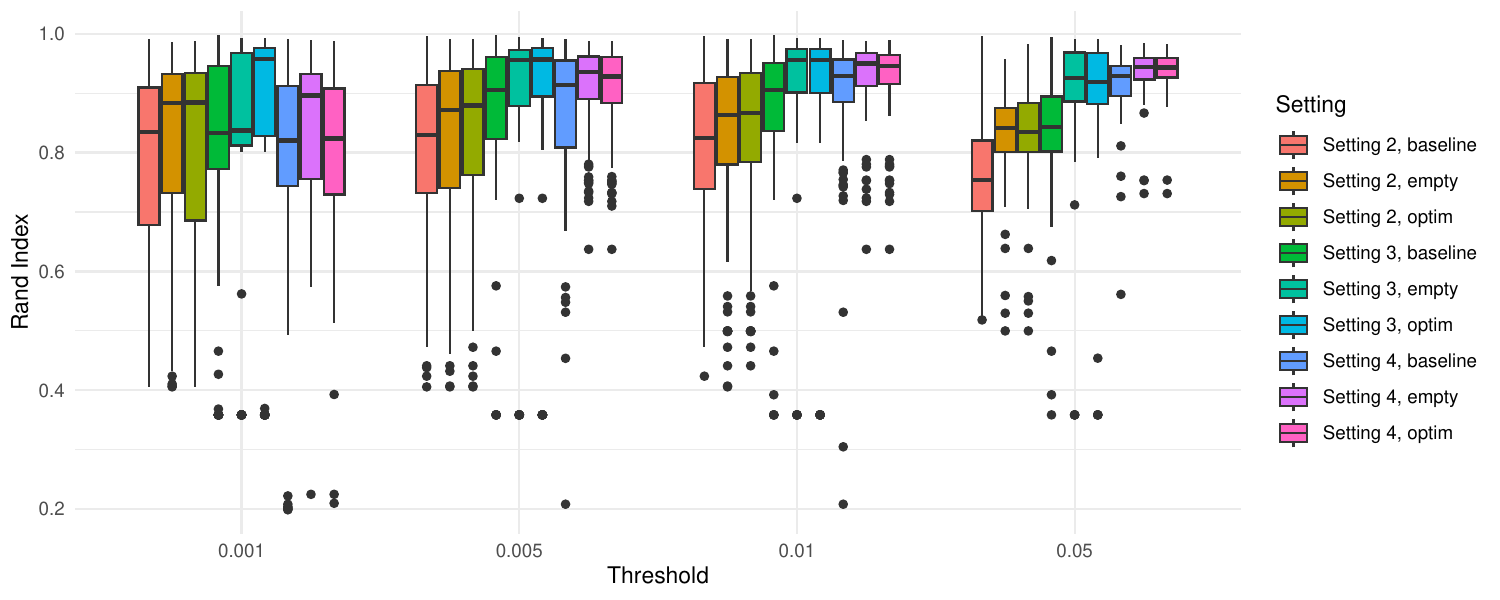}
    \caption{Boxplots of the Rand index comparing clusters identified through detected change points with the true clusters defined by the true change points, for different thresholds in the EMS framework using the \texttt{Llama} LLM with test statistic proposed in \citet{kuditipudi2023robust}.} \label{fig:ml3-EMS-rand_index-new}
\end{figure}

\begin{figure}
    \centering
    \includegraphics[width=\linewidth]{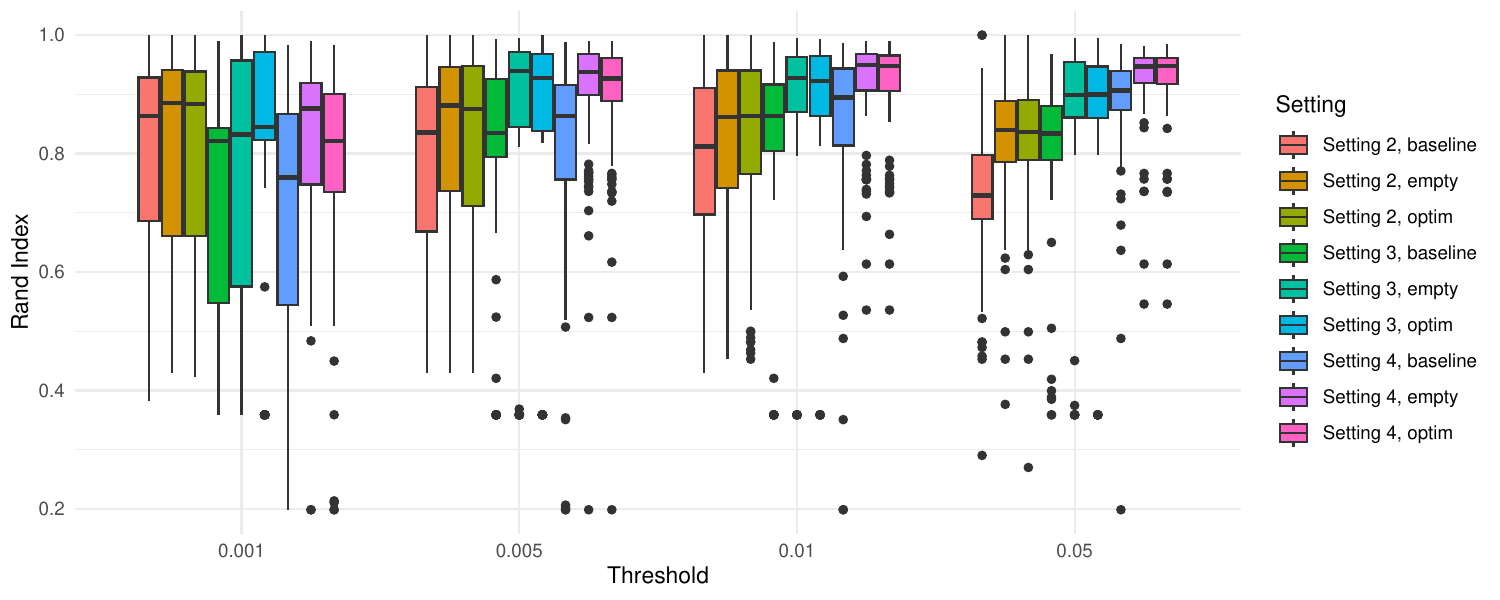}
    \caption{Boxplots of the Rand index comparing clusters identified through detected change points with the true clusters defined by the true change points, for different thresholds in the EMS framework using the \texttt{Mistral} LLM with test statistic proposed in \citet{kuditipudi2023robust}.} \label{fig:mt7-EMS-rand_index-new}
\end{figure}

\subsection{Comparison with Other Methods}
We also compare our method with several state-of-the-art methods, including \citet{xie2025debiasing}, \citet{li2024robust}, \citet{kashtan2023information}, and test statistics proposed in \citet{Aaronson2023}. Since most of these methods are designed for watermark detection rather than segmentation, we first describe below how we incorporated them into our framework.

\begin{enumerate}
    \item \citet{xie2025debiasing} propose using the Maximal Coupling method to debias the ``red–green list'' approach. In our framework, we apply a sliding window (as in our procedure for generating the $p$-value sequence) to obtain a sequence of test $p$-values. We then apply change-point detection to this sequence and evaluate performance using the Rand index. We refer to this method as \texttt{Coupling}.

    \item \citet{li2024robust} develop a mixture-model approach for watermark detection. In our setting, we again use a sliding window to generate a sequence of test statistics, apply change-point detection to this sequence, and compute the Rand index. We denote this method by \texttt{Tr-GoF}.

    \item \citet{kashtan2023information} propose a method that identifies sentences most likely written by humans, where the sentence boundaries directly correspond to change points in our framework. We therefore use their reported change points to compute the Rand index, referring to this method as \texttt{IT-LPPT}.
\end{enumerate}

Figure~\ref{fig:rand-comp} compares our method with three alternatives above. Here, \texttt{ours} denotes our method with an empty prompt, which is feasible in real applications. Across all settings, our method consistently outperforms the alternatives. In the first two settings, \texttt{IT-LPPT} performs better than both \texttt{Tr-GoF} and \texttt{Coupling}, whereas in the third setting, \texttt{Tr-GoF} and \texttt{Coupling} outperform \texttt{IT-LPPT}. The relatively weaker performance of \texttt{IT-LPPT} stems from its reliance on sentence boundaries, which do not align with the change-point structure in our scenarios. Meanwhile, \texttt{Tr-GoF} and \texttt{Coupling} were originally designed for watermark detection rather than segmentation, and thus require further adaptation for segmentation tasks rather than direct incorporation into our framework.

\begin{figure}
    \centering
    \includegraphics[width=0.95\linewidth]{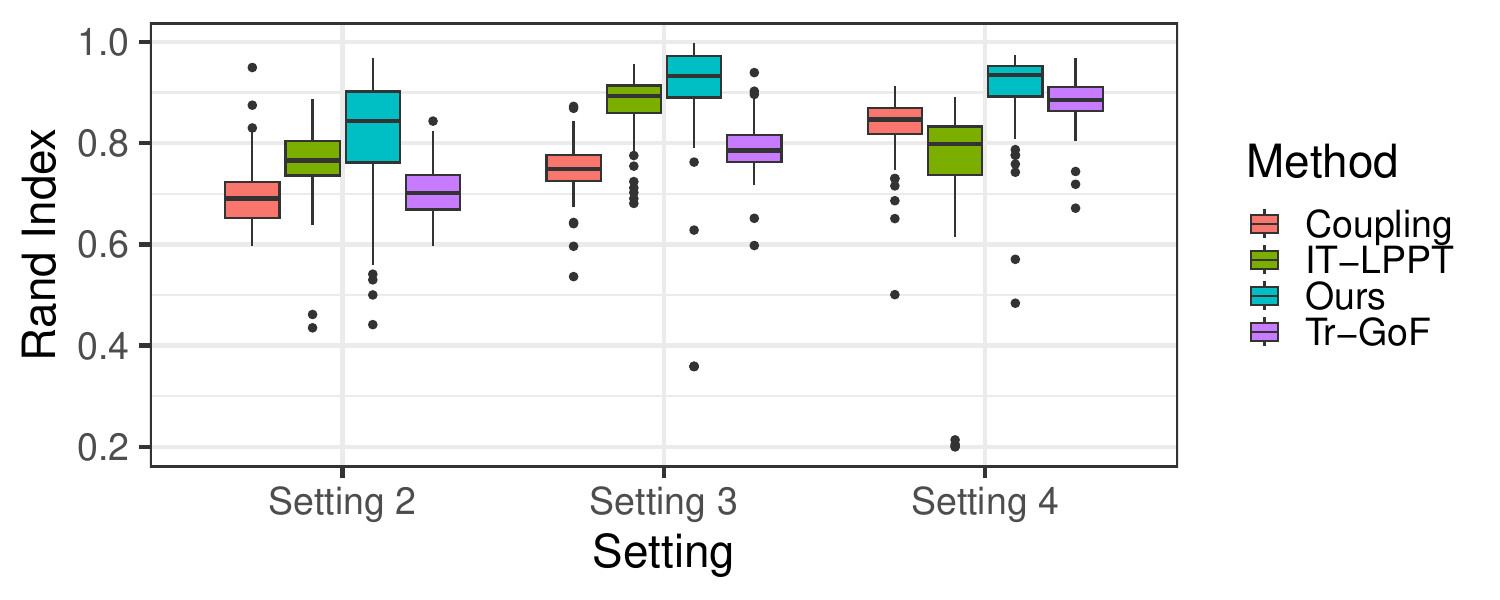}
    \caption{Comparison of four methods based on the Rand index.}
    \label{fig:rand-comp}
\end{figure}

\newpage
\section{Numerical Results for the ITS Method}\label{sec:numerical-results-ITS}
\subsection{Numerical Results for the ITS Method using the \texttt{Llama} Model}
Figure~\ref{fig:ml3-ITS-rand_index} shows the Rand index for different thresholds in the ITS framework using the \texttt{Llama} LLM. Finally, Figure~\ref{fig:ml3-ITS-pvalue} presents the $p$-value sequences obtained for the first 10 prompts using the \texttt{Llama} model with watermarks generated by the IS method.

\begin{figure}
    \centering
    \includegraphics[width=\linewidth]{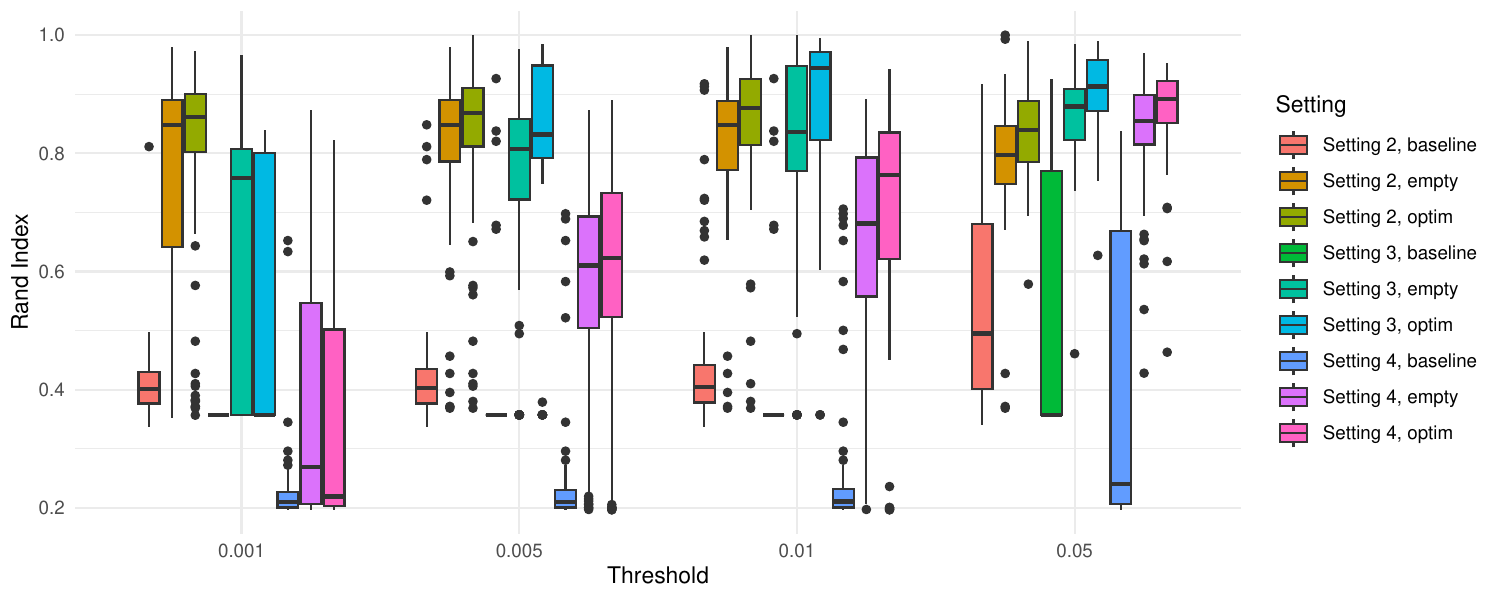}
    \caption{Boxplots of the Rand index for different thresholds in the ITS framework using the \texttt{Llama} LLM.}
    \label{fig:ml3-ITS-rand_index}
\end{figure}

\begin{figure}
    \centering
    \includegraphics[width=\linewidth]{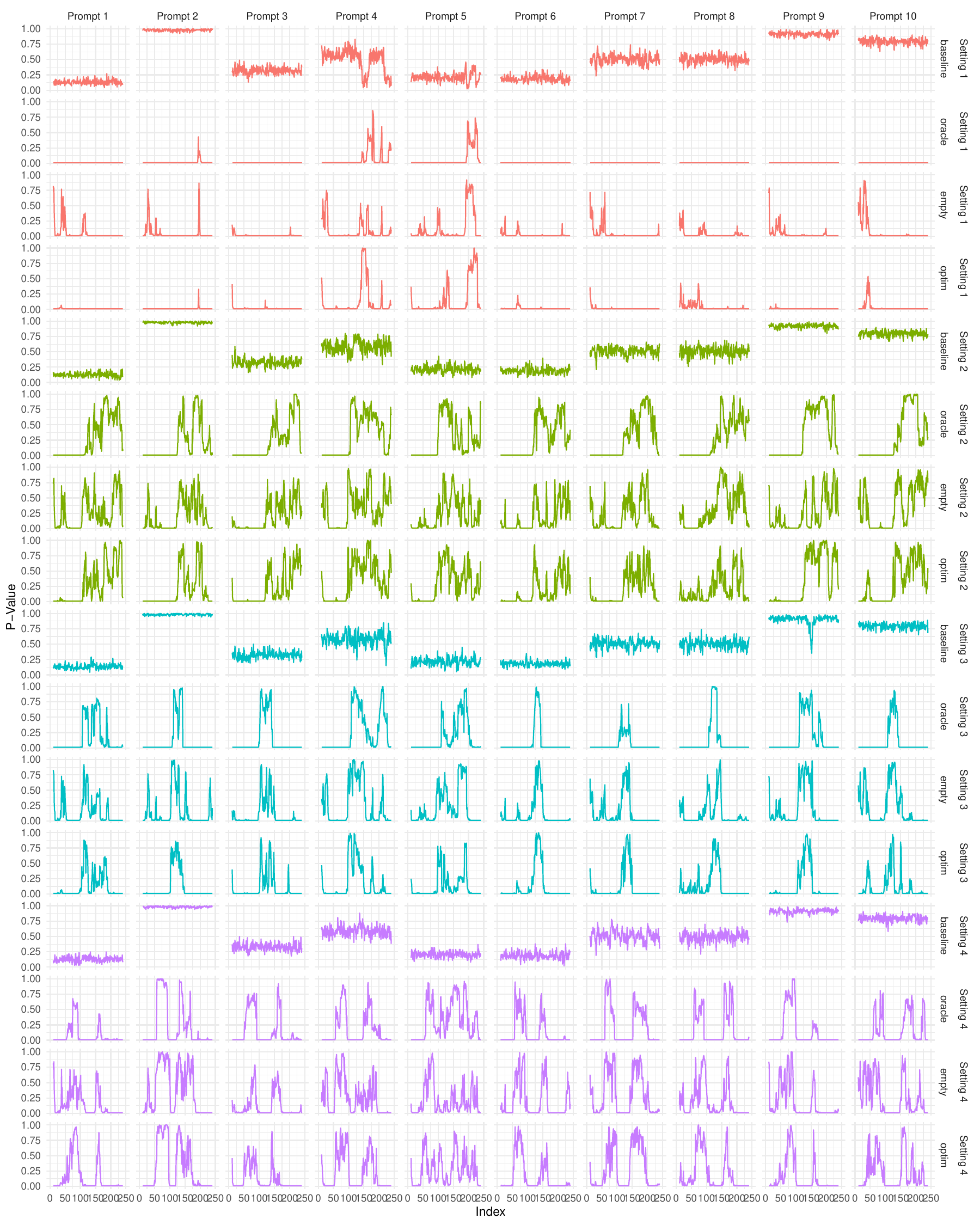}
    \caption{$p$-value sequence calculated using watermarked texts generated from the \texttt{Llama} LLM with the ITS method.}
    \label{fig:ml3-ITS-pvalue}
\end{figure}

In all settings, the proposed adaptive method outperforms the \texttt{baseline} method, demonstrating the advantages of the proposed approach.

\subsection{Numerical Results for the ITS Method using the \texttt{Mistral} Model}
Figure~\ref{fig:mt7-ITS-rand_index} shows the Rand index for different thresholds in the ITS framework using the \texttt{Mistral} LLM. Finally, Figure~\ref{fig:mt7-ITS-pvalue} presents the $p$-value sequences obtained for the first 10 prompts using the \texttt{Mistral} model with watermarks generated by the IS method.

\begin{figure}
    \centering
    \includegraphics[width=\linewidth]{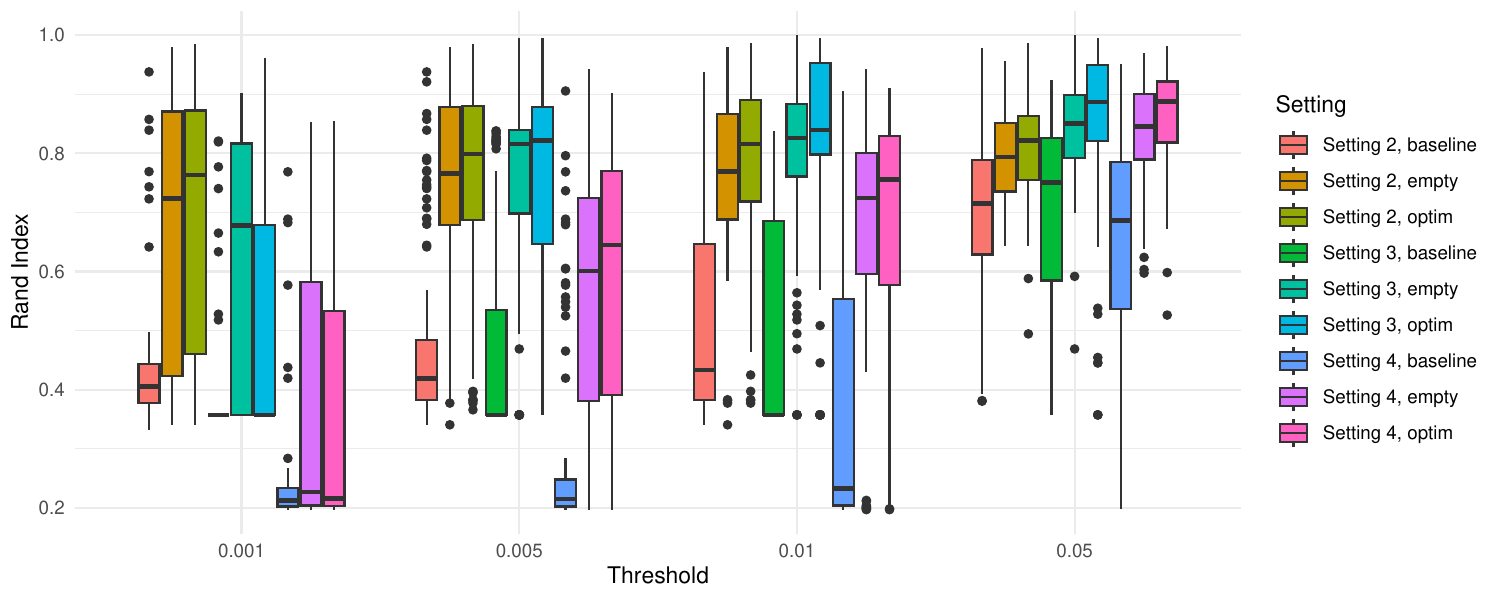}
    \caption{Boxplots of the Rand index for different thresholds in the ITS framework using the \texttt{Mistral} LLM.}
    \label{fig:mt7-ITS-rand_index}
\end{figure}

\begin{figure}
    \centering
    \includegraphics[width=\linewidth]{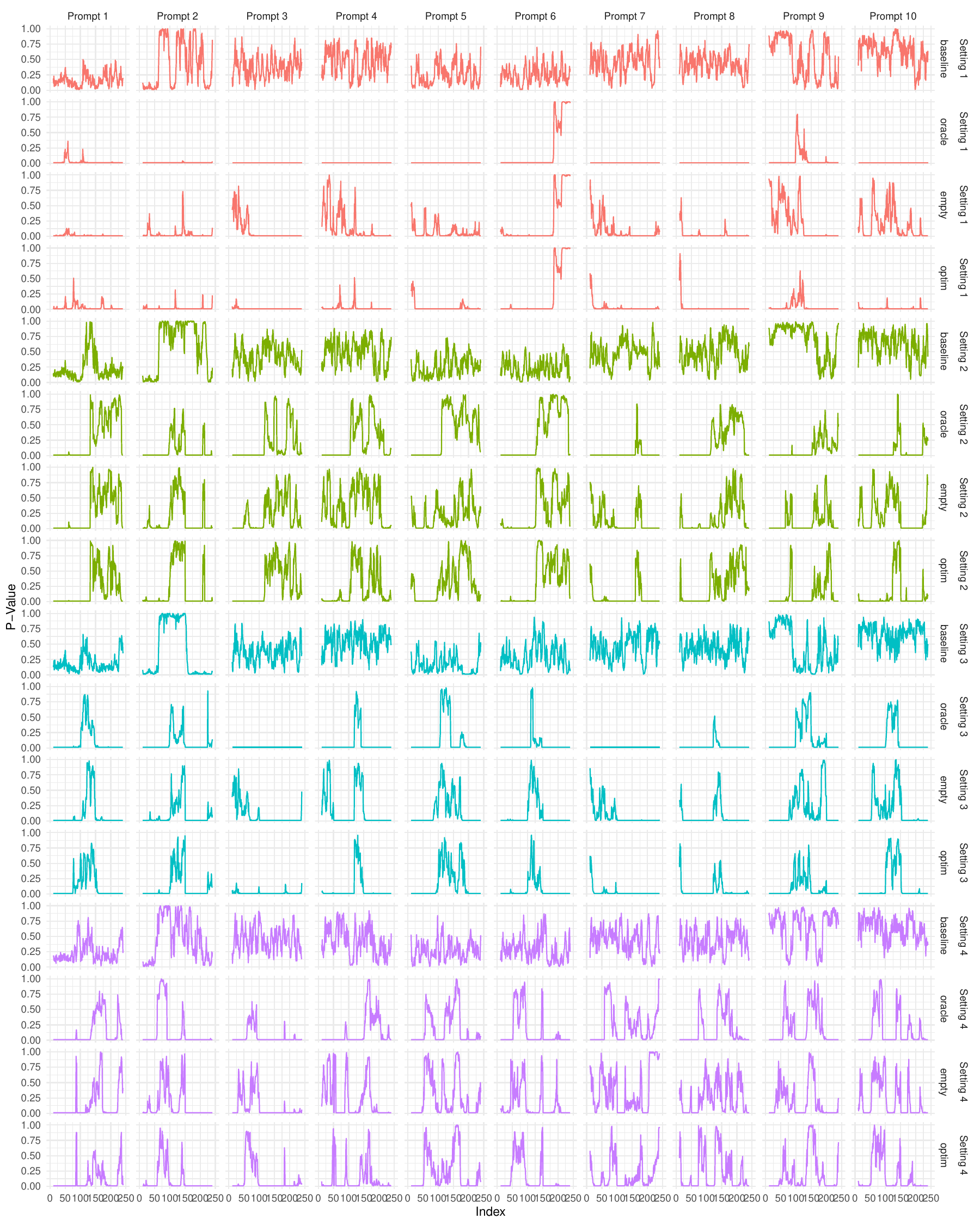}
    \caption{$p$-value sequence calculated using watermarked texts generated from the \texttt{Mistral} LLM with the ITS method.}
    \label{fig:mt7-ITS-pvalue}
\end{figure}

In all settings, the proposed adaptive method outperforms the \texttt{baseline} method, demonstrating the advantages of the proposed approach.

\newpage
\section{Choice of Hyperparameters}
First, to study the impact of the window size $B$, we fix $T=999$, $B' = 20$ and vary the value of $B$ in the set $\{10, 20, 30, 40, 50\}$. Table~\ref{tab:diff_B} shows the rand index value for each setting, with a higher rand index indicating better performance.

It is crucial to select an appropriate value for $B$. If $B$ is too small, the corresponding window may not contain enough data to reliably detect watermarks, as longer strings generally make the watermark more detectable. Conversely, if $B$ is excessively large, it might prematurely shift the detected change point locations, thus reducing the rand index. For instance, let us consider a scenario with 200 tokens where the first 100 tokens are non-watermarked, and the subsequent 100 are watermarked, with the true change point at index 101. Assuming our detection test is highly effective, then it will yield a p-value uniformly distributed over $[0,1]$ over a non-watermarked window and a p-value around zero over a window containing watermarked tokens. When $B = 50$, the window beginning at the 76th token contains one watermarked token, which can lead to a small p-value and thus erroneously indicate a watermark from the 76th token onwards. In contrast, if $B = 20$, the window starting at the 91st token will contain the first watermarked token, leading to a more minor error in identifying the change point location. The above phenomenon is the so-called edge effect, which will diminish as $B$ gets smaller.

The trade-off in the choice of window size is recognized in the time series literature. For instance, a common recommendation for the window size in time series literature is to set $B=Cn^{1/3}$, where $n$ is the sample size, as seen in Corollary 1 of \citep{lahiri1999theoretical}. Based on our experience, setting $B = \lfloor 3n^{1/3} \rfloor$ (for example, when $n=500$, $B=23$) often results in good finite sample performance. A more thorough investigation of the choice of $B$ is deferred to future research.

Our experimental results indicate that setting $B = B'$ is unnecessary. In practice, the sequence of $p$-values is $B$-dependent, with $p_i$ and $p_j$ being independent only if $|i - j| > B$. Consequently, we recommend using $B' = B$ to ensure that the block bootstrap adequately captures this dependence.

\begin{table}
    \centering
    \begin{tabular}{|c|ccccc|}
    \hline
    $B$ & 10 & 20 & 30 & 40 & 50\\
    \hline
    rand index & 0.8808 & 0.9429 &  0.9641 & 0.9570 &0.9243\\
    \hline
    \end{tabular}
    \caption{Results for different choices of $B$ when $T=999$.}
    \label{tab:diff_B}
\end{table}

\newpage
\section{Illustration of the change point detection algorithm}
Here, we present one partially watermarked text generated using the \texttt{Llama} model with the EMS method. The prompt used to generate the text is marked in the upper box, and the generated text is marked below in Figure~\ref{fig:newspaper-example}.

\begin{figure}
    \centering
    \includegraphics[width=\linewidth]{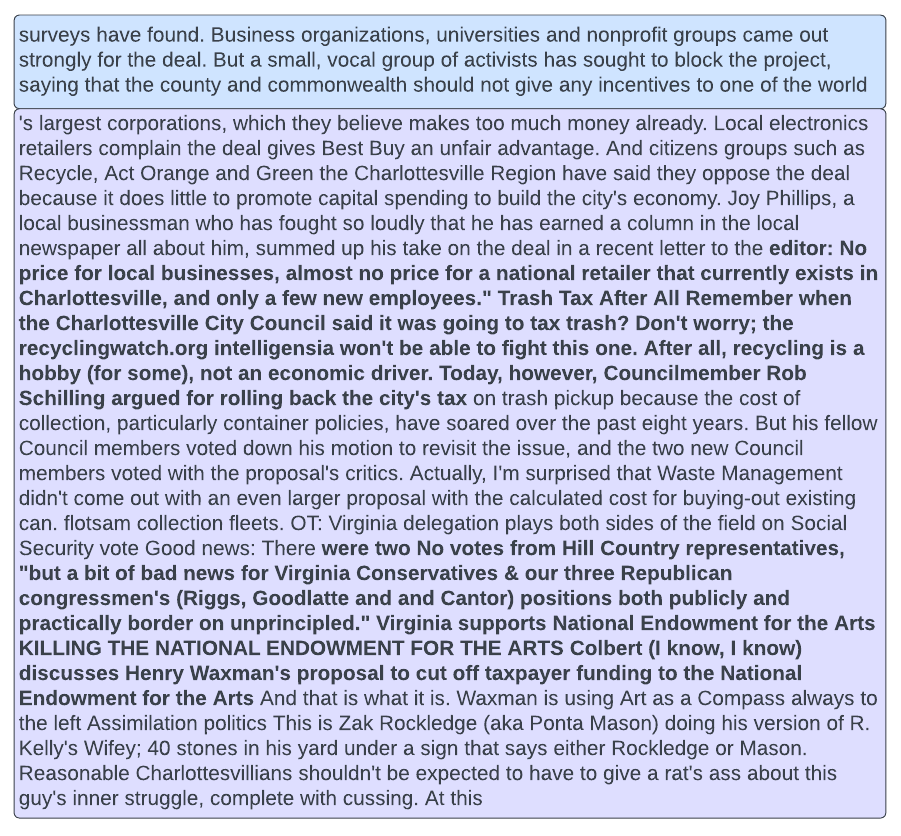}
    \caption{Example prompt and generated text from the \texttt{Llama} model with watermarked texts highlighted with bold font.}
    \label{fig:newspaper-example}
\end{figure}

The change points are located at the tokens: \emph{editor}, \emph{on trash pickup}, \emph{were two No votes} and \emph{And that is what it is}, corresponding to the change points located at $100, 200, 300$ and $400$ after tokenization. The second and fourth segments: ``editor ...... city's tax'' and ``were two No votes ...... for the Arts'' are watermarked with the EMS method. The generated SeedBS intervals can be found in Figure~\ref{fig:seedbs-example}, and the $p$-values sequence can be found in Figure~\ref{fig:newspaper-pvalue-example}. According to Algorithm~\ref{algorithm:seedbs-not}, segments with significant change points for the $p$-value distributions are marked in red in Figure~\ref{fig:not-example} and by NOT, the final change points are marked using red dots, near $100, 200, 300$ and $400$, corresponding to the texts: \emph{a recent letter}, \emph{on trash pickup}, \emph{There were two} and \emph{to cut off}, respectively.

\begin{figure}[ht]
    \centering
    \includegraphics[width=\linewidth]{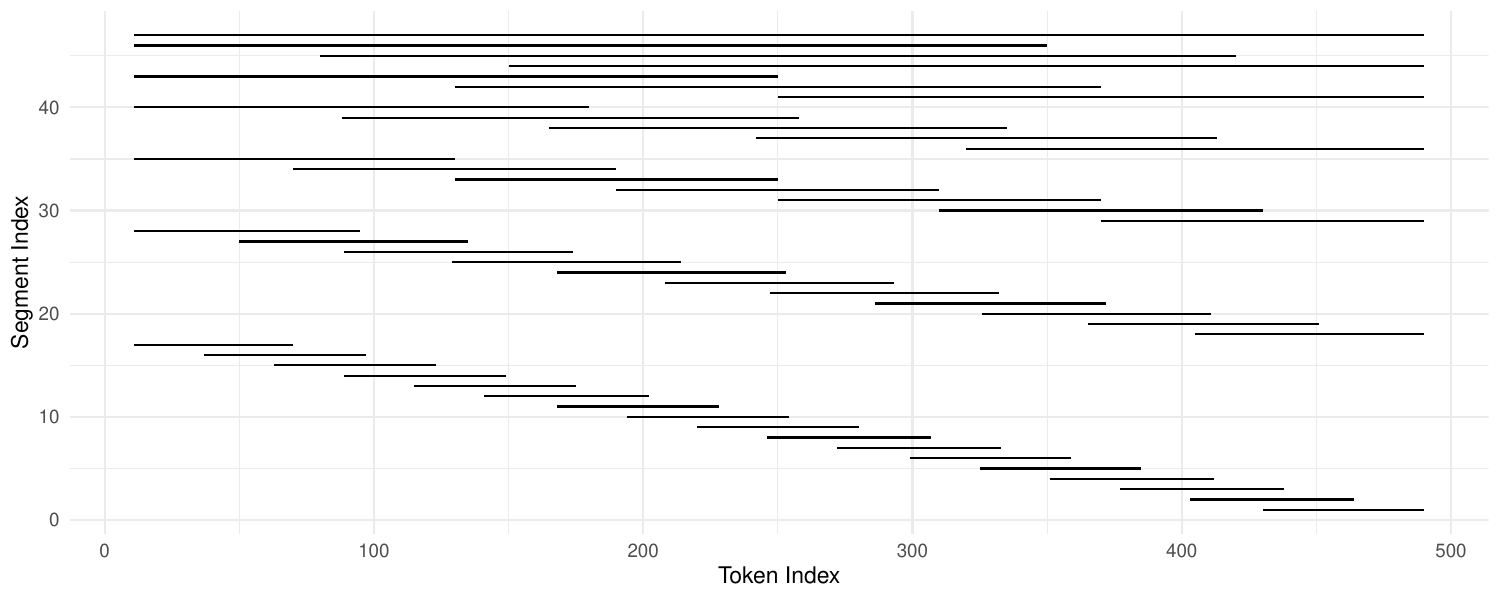}
    \caption{SeedBS intervals in Algorithm~\ref{algorithm:seedbs-not} with $I_1=(0, 500]$ and the decay parameter $a=\sqrt{2}$.}
    \label{fig:seedbs-example}
\end{figure}

\begin{figure}
    \centering
    \includegraphics[width=\linewidth]{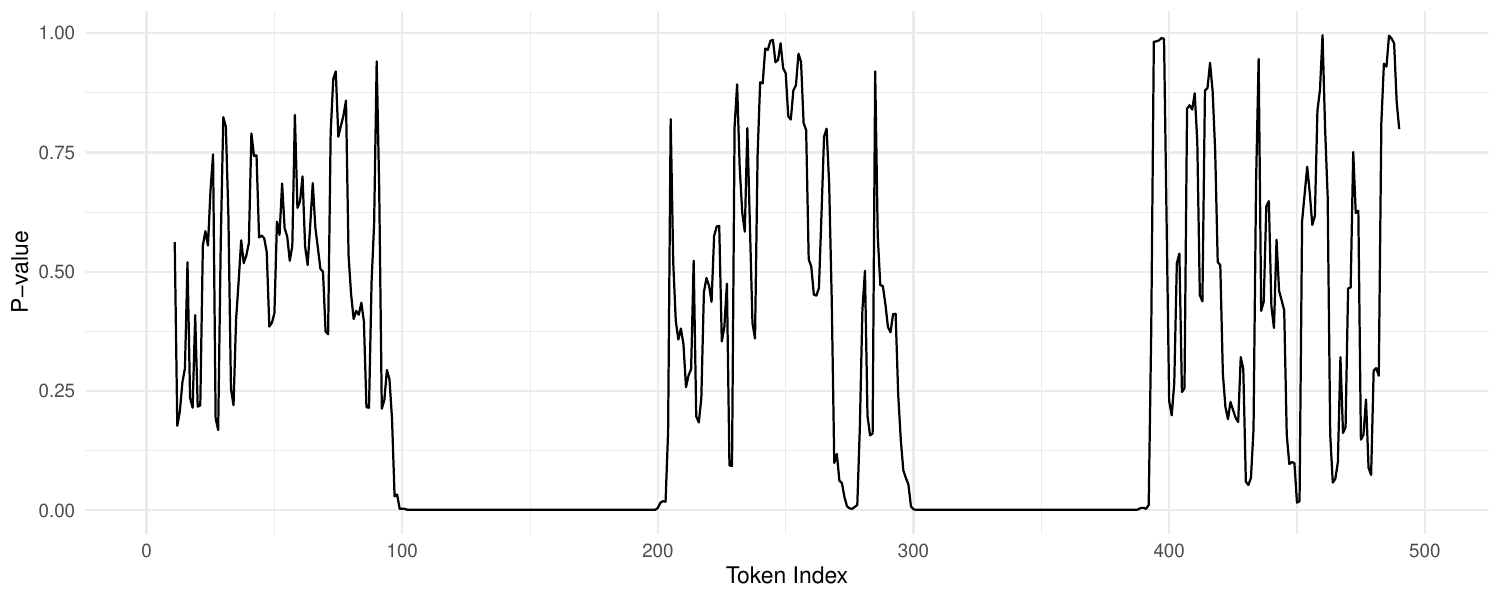}
    \caption{$p$-value sequence from Setting 4 in Section~\ref{sec:segmenting-watermarked-texts} with four change points located at $100, 200, 300$ and $400$. The second and fourth segments are watermarked.}
    \label{fig:newspaper-pvalue-example}
\end{figure}

\begin{figure}
    \centering
    \includegraphics[width=\linewidth]{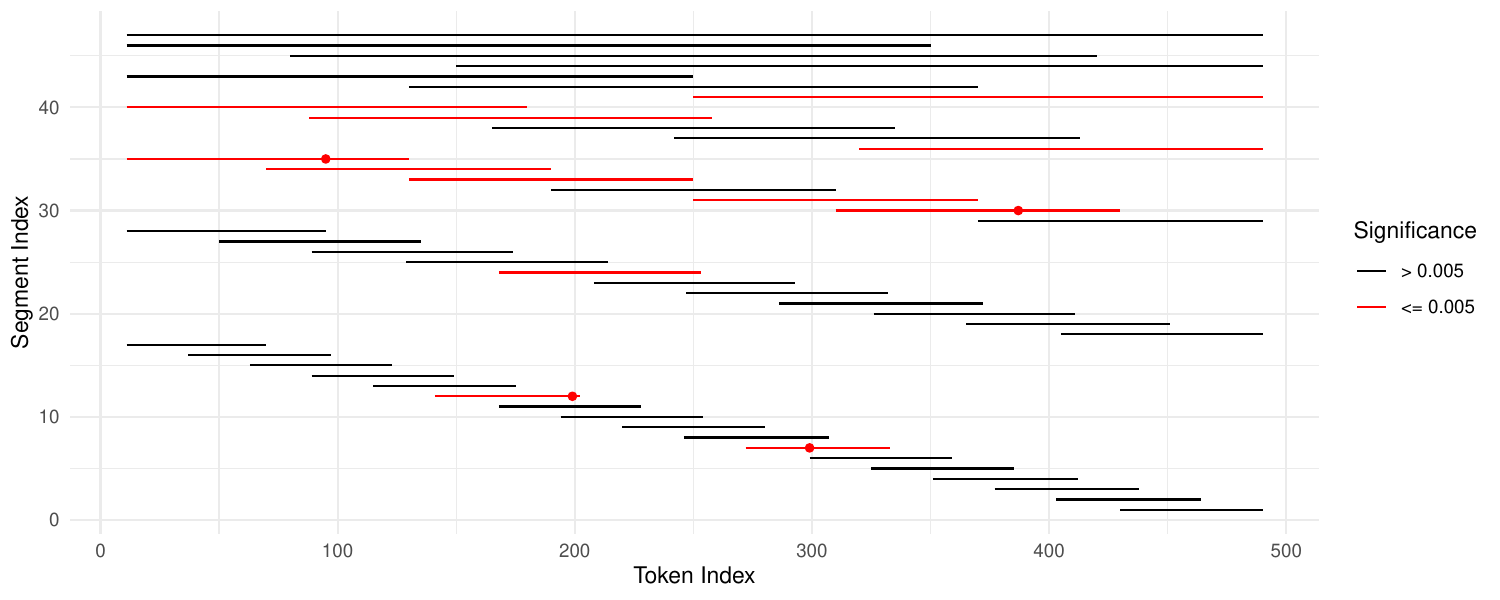}
    \caption{Segments with significant change points are marked in red, and the final estimated change points are marked with red dots. }
    \label{fig:not-example}
\end{figure}

\end{document}